\documentclass{article} % DO NOT CHANGE THIS
\usepackage{nac_preprint}
\usepackage{times}  % DO NOT CHANGE THIS
\usepackage{helvet}  % DO NOT CHANGE THIS
\usepackage{courier}  % DO NOT CHANGE THIS
\usepackage[hyphens]{url}  % DO NOT CHANGE THIS
\usepackage{graphicx} % DO NOT CHANGE THIS
\usepackage{caption} % DO NOT CHANGE THIS AND DO NOT ADD ANY OPTIONS TO IT

\usepackage{algorithm}
\usepackage{algorithmic}
\usepackage{xcolor}
\usepackage{amsmath, amssymb, amsthm}
\usepackage{multirow}
\usepackage{subcaption}
\theoremstyle{plain}
\usepackage{graphicx}
\newtheorem{theorem}{Theorem}[section]

\theoremstyle{definition}
\newtheorem{definition}[theorem]{Definition}

\theoremstyle{remark}

\usepackage{newfloat}
\usepackage{listings}
\DeclareCaptionStyle{ruled}{labelfont=normalfont,labelsep=colon,strut=off} % DO NOT CHANGE THIS
\floatstyle{ruled}
\newfloat{listing}{tb}{lst}{}
\floatname{listing}{Listing}
\title{Exploring Learnability in Memory-Augmented Recurrent Neural Networks: Precision, Stability, and Empirical Insights}
\author{Shrabon Das %\thanks{ Use footnote for providing further information about author (webpage, alternative address)---\emph{not} for acknowledging funding agencies.  Funding acknowledgements go at the end of the paper.} 
\\
University of South Florida \\
Tampa, FL 33620, USA \\
\texttt{\{das157\}@usf.edu} \\
\And
Ankur Mali \\
University of South Florida \\
Tampa, FL 33620, USA \\
\texttt{\{ankurarjunmali\}@usf.edu} \\
}
\usepackage{bibentry}
\begin{document}

\maketitle

\begin{abstract}
This study investigates the learnability of memory-less and memory-augmented Recurrent Neural Networks (RNNs) with deterministic and non-deterministic stacks, which are theoretically equivalent to Pushdown Automata in terms of expressivity. However, empirical evaluations reveal that these models often fail to generalize on longer sequences, particularly when learning context-sensitive languages, suggesting they rely on precision rather than mastering symbolic grammar rules. Our experiments examined fully trained models and models with various frozen components: the controller, the memory, and only the classification layer. While all models showed similar performance on training validation, the model with frozen memory achieved state-of-the-art performance on the Penn Treebank (PTB) dataset, reducing the best overall test perplexity from 123.5 to 120.5—a gain of approximately 1.73\%. When tested on context-sensitive languages, models with frozen memory consistently outperformed others on small to medium test sets. Notably, well-trained models experienced up to a 60\% performance drop on longer sequences, whereas models with frozen memory retained close to 90\% of their initial performance. Theoretically, we explain that freezing the memory component enhances stability by anchoring the model's capacity to manage temporal dependencies without constantly adjusting memory states. This approach allows the model to focus on refining other components, leading to more robust convergence to optimal solutions. These findings highlight the importance of designing stable memory architectures and underscore the need to evaluate models on longer sequences to truly understand their learnability behavior and limitations. The study suggests that RNNs may rely more on precision in data processing than on internalizing grammatical rules, emphasizing the need for improvements in model architecture and evaluation methods.
\end{abstract}

\section{Introduction}
Recurrent Neural Networks (RNNs) have been foundational in sequence modeling due to their ability to capture temporal dependencies. Architectures such as Elman RNNs, Gated Recurrent Units (GRUs), and Long Short-Term Memory networks (LSTMs) \cite{lstmcfg} are widely used in applications like speech recognition, machine translation, and time-series analysis. However, these models are constrained by their fixed memory capacity, limiting them to recognizing regular languages when implemented with finite precision \cite{merrill2020formal, mali2023computational}.

To enhance the computational capabilities of RNNs, researchers have explored augmenting them with external memory structures like stacks \cite{mali2020neural, stogin2024provably, graves2016hybrid, joulin2015inferring, grefenstette2015learning, dusell2020learning, mali2021recognizing}. This approach extends the expressivity of RNNs to context-free languages (CFLs) \cite{hopcroft2001introduction}, which are crucial in applications like natural language processing (NLP) where hierarchical structures are prevalent.

Memory-augmented models have demonstrated significant improvements in recognizing complex formal languages by simulating operations similar to Pushdown Automata (PDA). These models can process deterministic and nondeterministic context-free languages, thereby expanding the class of languages they can handle compared to traditional RNNs. However, while their theoretical expressivity is well-established, the learnability of these systems remains an open challenge. Learnability here refers to the model’s ability to reliably generalize from training data to unseen inputs, particularly with complex recursive patterns.

Despite their theoretical advantages, empirical studies indicate that many memory-augmented models struggle to generalize to longer sequences. This instability is likely due to the dynamic nature of memory manipulation and finite precision constraints, leading to performance degradation on extended sequences. Learnability is closely tied to stability: unstable systems are prone to accumulating errors over time, particularly as sequence lengths increase, resulting in unpredictable behavior. Understanding stability conditions is therefore crucial for enhancing learnability. In this work, we investigate how different configurations, such as freezing the RNN controller while training the memory, impact stability and performance compared to fully trainable but unstable setups. Stability, in this context, refers to a model's ability to maintain consistent performance across varying sequence lengths and input complexities. A stable model efficiently generalizes from training data while remaining robust to variations in input and computational precision. Stability is key to ensuring reliable learnability for complex tasks.

Key aspects of stability include:
\begin{itemize}
\item \textbf{Consistency Across Sequence Lengths:} Ensuring stable performance on both short and long sequences.
\item \textbf{Robustness to Variability:} Handling variations in input and precision without significant performance degradation.
\item \textbf{Reliable Memory Manipulation:} Maintaining correct memory operations even under finite precision.
\item \textbf{Enhanced Generalization:} Promoting better generalization across different tasks and datasets.
\end{itemize}

This study addresses the critical relationship between stability and learnability in stack-augmented RNNs by providing theoretical and empirical insights into when these models succeed and when they fail. Our key contributions include:

\begin{itemize}
\item Theoretical analysis of stability and instability conditions for stack-augmented RNNs, showing that configurations with a frozen RNN controller but a trainable stack can outperform fully trained yet unstable models.
\item Derivation of error bounds for unstable systems, illustrating how models that initially perform well on slightly longer sequences may eventually converge to random guessing as sequence lengths increase.
\item Analysis of learnability under varying conditions, demonstrating how stability plays a crucial role in effective generalization to unseen sequences.
\item A framework for understanding the impact of machine precision on memory operations, emphasizing the need for stable memory manipulation to ensure robustness.
\item Empirical validation through experiments on context-sensitive languages and the Penn Treebank dataset, demonstrating alignment with theoretical findings.
\end{itemize}

By focusing on stability, we bridge the gap between theoretical expressivity and practical learnability, offering insights into designing more robust architectures capable of handling complex language structures while maintaining stable performance across diverse conditions. Understanding stability through the lens of machine precision and error growth provides strategies for improving the learnability of stack-augmented RNNs in real-world applications

\section{Background}

This section describes the basic concepts discussed in this work, focusing on Pushdown Automata (PDA) and stack-augmented Recurrent Neural Networks (RNNs). We explore how these models differ in expressivity and how they recognize formal languages. \textbf{Pushdown Automaton (PDA):} A PDA is a computational model capable of recognizing context-free languages using a stack as auxiliary memory. Formally, a PDA is defined as:
\[
M = (Q, \Sigma, \Gamma, \delta, q_0, Z_0, F),
\]
where \( Q \) represents states, \( \Sigma \) the input alphabet, \( \Gamma \) the stack alphabet, and \( \delta \) the transition function. The PDA reads an input string \( x \in \Sigma^* \), updates its stack and states according to \( \delta \), and either accepts or rejects \( x \) based on the final state and stack configuration.

Stack-Augmented Recurrent Neural Network (RNN with Memory): To extend RNNs’ expressivity to context-free languages, they can be augmented with a stack. A stack-augmented RNN is defined as:
\[
f_\theta = (\theta_c, \theta_m),
\]
where \( \theta_c \) governs the RNN controller, and \( \theta_m \) controls stack operations (push, pop, no-op). At each time step, the model:
\begin{enumerate}
    \item Updates the hidden state:
    \[
    h_t = f(W_h h_{t-1} + W_x x_t + b).
    \]
    \item Chooses a stack operation:
    \[
    \text{Stack}_{t+1} = \delta_s(h_t, \text{Stack}_t),
    \]
    where \( \delta_s \) is a learned function.
    \item Produces the output based on both the hidden state and stack:
    \[
    y_t = g(W_y h_t + W_s \text{Stack}_t + b_y).
    \]
\end{enumerate}

Comparison of Expressivity: While a PDA directly handles context-free languages using a stack, a stack-augmented RNN approximates this behavior through learned stack operations. The primary challenge lies in maintaining stability and generalization as sequence complexity increases, particularly for long or deeply nested structures.

\section{Methodology}

In this section, we first formally define the stability condition of memory less model, where stability is defined as follows
\begin{definition}[Stability]
Let $L$ be a formal language accepted by a discrete state machine $M$, such as a finite automaton. A sequence model $f$ is said to be \textit{stable} with respect to $L$ if, for any input sequence $x = (x_1, x_2, \dots, x_T)$ of arbitrary length $T \in \mathbb{N}$:

\begin{itemize}
    \item There exists a mapping $\phi: \text{States}(M) \to \text{States}(f)$ such that the state transitions of $f$ are functionally equivalent to those of $M$. Formally, for each state $q \in \text{States}(M)$ and corresponding state $\phi(q) \in \text{States}(f)$:
    \[
    \delta_M(q, x_i) = q' \iff \delta_f(\phi(q), x_i) = \phi(q'),
    \]
    where $\delta_M$ and $\delta_f$ denote the transition functions for $M$ and $f$, respectively, and $x_i$ is an input symbol.

    \item The acceptance condition for $L$ is preserved by $f$. That is, for any sequence $x \in \Sigma^*$:
    \[
    x \in L \iff f(x) = 1,
    \]
    where $\Sigma$ is the input alphabet, and $f(x) = 1$ indicates acceptance by the model $f$.

    \item The behavior of $f$ remains consistent for sequences of arbitrary length. For any two sequences $x^{(1)}$ and $x^{(2)}$ of lengths $T_1$ and $T_2$, respectively, with $T_1 \neq T_2$:
    \[
    M(x^{(1)}) = f(x^{(1)}) \quad \text{and} \quad M(x^{(2)}) = f(x^{(2)}),
    \]
    where $M(x)$ and $f(x)$ denote the outputs of the discrete state machine and the model, respectively.

    \item Stability is maintained as $T \to \infty$, ensuring that $f$ does not exhibit performance degradation or erratic behavior for long sequences.
\end{itemize}

In essence, a model $f$ is stable if it is functionally equivalent to the discrete state machine $M$ that defines the language $L$, preserving both the state transitions and acceptance conditions across sequences of arbitrary length.
\end{definition}
Similarly stability of memory augmented RNN is formally defined as follows: 
\begin{definition}[Stability of a Pushdown Automaton (PDA) Model]
Let $L$ be a formal language accepted by a Pushdown Automaton (PDA) $M = (Q, \Sigma, \Gamma, \delta, q_0, Z_0, F)$, where:
\begin{itemize}
    \item $Q$ is the set of states,
    \item $\Sigma$ is the input alphabet,
    \item $\Gamma$ is the stack alphabet,
    \item $\delta: Q \times (\Sigma \cup \{\epsilon\}) \times \Gamma \to Q \times \Gamma^*$ is the transition function,
    \item $q_0 \in Q$ is the initial state,
    \item $Z_0 \in \Gamma$ is the initial stack symbol,
    \item $F \subseteq Q$ is the set of accepting states.
\end{itemize}

A sequence model $f$ is said to be \textit{stable} with respect to $L$ if, for any input sequence $x = (x_1, x_2, \dots, x_T)$ of arbitrary length $T \in \mathbb{N}$:

\begin{itemize}
    \item \textbf{State Stability:} There exists a mapping $\phi: Q \to \text{States}(f)$ such that the state transitions of $f$ are functionally equivalent to those of the PDA $M$. Formally, for each state $q \in Q$ and corresponding state $\phi(q) \in \text{States}(f)$:
    \[
    \delta(q, x_i, \gamma) = (q', \gamma') \iff \delta_f(\phi(q), x_i, \gamma) = (\phi(q'), \gamma'),
    \]
    where $\gamma \in \Gamma$ is the top symbol on the stack and $\delta_f$ represents the transition function of the model $f$.

    \item \textbf{Stack Stability:} For the stack operations in $f$ to be stable, the stack manipulation should be equivalent to that of the PDA $M$. Let $\text{Stack}_M(t)$ and $\text{Stack}_f(t)$ represent the stack contents at time $t$ for $M$ and $f$, respectively. The model $f$ is stable if:
    \[
    \text{Stack}_M(t) = \text{Stack}_f(t) \quad \text{for all } t.
    \]
    This ensures that push and pop operations are executed consistently and that the stack remains in sync with the PDA for arbitrary input sequences.

    \item \textbf{Acceptance Condition:} The model $f$ correctly accepts or rejects sequences according to $L$. That is, for any sequence $x \in \Sigma^*$:
    \[
    x \in L \iff f(x) = 1,
    \]
    where $f(x) = 1$ indicates acceptance by the model $f$.

    \item \textbf{Behavioral Consistency Across Sequence Lengths:} For any two sequences $x^{(1)}$ and $x^{(2)}$ of lengths $T_1$ and $T_2$, respectively, with $T_1 \neq T_2$:
    \[
    M(x^{(1)}) = f(x^{(1)}) \quad \text{and} \quad M(x^{(2)}) = f(x^{(2)}),
    \]
    where $M(x)$ and $f(x)$ denote the outputs of the PDA and the model, respectively.

    \item \textbf{Long-Term Stability:} The model $f$ remains stable for arbitrarily long sequences, ensuring that its behavior and stack operations do not degrade or become inconsistent as $T \to \infty$.
\end{itemize}

In summary, a model $f$ is stable if it is functionally equivalent to the PDA $M$ that defines the language $L$, preserving both the state transitions and stack operations across sequences of arbitrary length.
\end{definition}

\subsection{Theoretical Analysis of Stability of Memory Augmented Neural Network}
In this section, we describe the analytical approach used to understand the behavior of stack-augmented Recurrent Neural Networks (RNNs) when processing formal languages. Our analysis is structured around a series of theorems that establish key properties such as error growth, stability, and the comparative performance of different configurations of stack-augmented RNNs.

First, we formalize the relationship between stability and model performance by evaluating the impact of stack operations and state transitions. Stability is crucial for ensuring consistent performance across sequences of varying lengths, and the presented theorems highlight how instability leads to unbounded error growth, ultimately converging to random guessing or worse. Formally we can show the existence of stable memory augmented RNN as follows:

\begin{theorem}[Stability of Stack-Augmented RNNs]
Let $L$ be a formal language recognized by a Pushdown Automaton (PDA) $M = (Q, \Sigma, \Gamma, \delta, q_0, Z_0, F)$. Consider a Recurrent Neural Network (RNN) $f$ augmented with a stack that models the PDA $M$. The RNN $f$ is said to be stable if it satisfies the following conditions:
\begin{enumerate}
    \item The state transitions of $f$ are functionally equivalent to those of the PDA $M$.
    \item The stack operations of $f$ (push, pop, no-op) are consistent with those of $M$.
    \item The outputs of $f$ and $M$ are equivalent for sequences of different lengths.
    \item As the sequence length $T \to \infty$, the error between $f$ and $M$ approaches zero.
\end{enumerate}

The RNN $f$ is unstable if any of the conditions above are violated.
\end{theorem}

\begin{proof}
1. Let $\phi: Q \to \text{States}(f)$ be a mapping between the states of the PDA $M$ and the states of the RNN $f$. The state transition function of $M$ is defined as:
\[
\delta: Q \times (\Sigma \cup \{\epsilon\}) \times \Gamma \to Q \times \Gamma^*,
\]
where $\Sigma$ is the input alphabet, $\Gamma$ is the stack alphabet, and $\delta$ is the transition function of $M$. The transition function of $f$, denoted as $\delta_f$, should satisfy:
\[
\delta(q, x_i, \gamma) = (q', \gamma') \iff \delta_f(\phi(q), x_i, \gamma) = (\phi(q'), \gamma'),
\]
for all states $q \in Q$, input symbols $x_i \in \Sigma$, and stack symbols $\gamma \in \Gamma$. If such a mapping $\phi$ exists and is consistent across all transitions, then the state transitions of $f$ are stable. The existence of this mapping ensures that for any input sequence, the behavior of $f$ mirrors the PDA $M$, maintaining state stability.

2. The stack operations in $f$ must replicate those of $M$ across all time steps. For any input sequence $x = (x_1, \dots, x_T)$ and corresponding stack operations $\gamma \in \Gamma$, let $\text{Stack}_M(t)$ and $\text{Stack}_f(t)$ denote the stack contents at time $t$ for $M$ and $f$, respectively. Stack stability requires:
\[
\text{Stack}_M(t) = \text{Stack}_f(t) \quad \forall t \in \{1, 2, \dots, T\}.
\]
If this condition holds, then the push, pop, and no-op operations of $f$ are consistent with those of $M$. Inconsistent stack behavior leads to errors in sequence recognition, causing instability.

3. Consider two input sequences $x^{(1)}$ and $x^{(2)}$ with lengths $T_1 \neq T_2$. Stability requires that the outputs of $f$ and $M$ are functionally equivalent:
\[
M(x^{(1)}) = f(x^{(1)}) \quad \text{and} \quad M(x^{(2)}) = f(x^{(2)}).
\]
This behavioral consistency across sequence lengths follows from stable state transitions and stack operations, ensuring that $f$ correctly processes sequences of varying lengths. If $f$ exhibits degradation in performance for longer sequences, the model is unstable.

4. As sequence length $T \to \infty$, stability requires:
\[
\lim_{T \to \infty} \text{Error}(f, M) = 0,
\]
where $\text{Error}(f, M)$ measures the divergence between the outputs of $f$ and $M$. If this condition holds, the long-term stability of $f$ is guaranteed. Any error accumulation as $T$ increases would indicate instability in either state transitions or stack operations, leading to performance degradation.

Thus the RNN $f$ is stable if all four conditions are met. If any of these conditions fail, the model becomes unstable, resulting in incorrect behavior or output divergence from the PDA $M$.
\end{proof}

Next, we formally prove that unstable memory-augmented RNN, after arbitrary steps, will eventually become equivalent to a random network, hampering its generalization. Formally we prove following theorem 

\begin{theorem}
Let $L$ be a formal language recognized by a Pushdown Automaton (PDA) $M = (Q, \Sigma, \Gamma, \delta, q_0, Z_0, F)$. Consider a Recurrent Neural Network (RNN) $f_\theta$ augmented with a stack and parameterized by $\theta$. Let $f_{\text{random}}$ be a randomly initialized network with parameters $\theta_{\text{random}}$. If the RNN $f_\theta$ is unstable, then the expected loss of $f_\theta$ is equivalent to the expected loss of $f_{\text{random}}$, i.e.,

\[
\mathbb{E} \left[ \text{Loss}(f_\theta(x), y) \right] \approx \mathbb{E} \left[ \text{Loss}(f_{\text{random}}(x), y) \right],
\]

where $x$ is the input, $y$ is the target output, and $\text{Loss}(\cdot, \cdot)$ is a suitable loss function.
\end{theorem}

\begin{proof}
Let $M(x)$ represent the correct output for input $x$ according to the PDA $M$, and let $f_\theta(x)$ represent the output of the RNN $f_\theta$ for the same input.

1. \textbf{Instability and Divergence from True Dynamics}:

    For an unstable RNN $f_\theta$, instability implies that the state transitions and stack operations diverge from those of the PDA $M$. Specifically, for sufficiently large input sequences $x$ with length $T$, the output of $f_\theta(x)$ diverges from the correct output $M(x)$:
    \[
    \lim_{T \to \infty} \left| f_\theta(x) - M(x) \right| \to \infty.
    \]
    The divergence grows as the sequence length increases, leading to increasingly erroneous outputs. In such cases, $f_\theta$ fails to learn the correct language $L$ and instead exhibits behavior that is statistically uncorrelated with $M(x)$.

2. \textbf{Equivalence to Random Network Behavior}:

    For a randomly initialized network $f_{\text{random}}$, the output behavior is not correlated with the target output $M(x)$. The expected loss for a random network is:
    \begin{align*}
    \mathbb{E} \left[ \text{Loss}(f_{\text{random}}(x), y) \right] &= \int_{x \in \Sigma^*} \text{Loss}(f_{\text{random}}(x), y) P(x) \, dx \\
    &\approx \text{Constant}.
\end{align*}
%     \[
%     \mathbb{E} \left[ \text{Loss}(f_{\text{random}}(x), y) \right] =
%     \int_{x \in \Sigma^*} \text{Loss}(f_{\text{random}}(x), y) P(x) \, dx 
%     \approx \text{Constant},
% \]
%     \[
%     \mathbb{E} \left[ \text{Loss}(f_{\text{random}}(x), y) \right] = \int_{x \in \Sigma^*} \text{Loss}(f_{\text{random}}(x), y) P(x) \, dx 
%     \approx \text{Constant},
% \]

    % \[
    % \mathbb{E} \left[ \text{Loss}(f_{\text{random}}(x), y) \right] = \int_{x \in \Sigma^*} \text{Loss}(f_{\text{random}}(x), y) P(x) \, dx \approx \text{Constant},
    % \]
    where $P(x)$ is the distribution over the input space $\Sigma^*$.

    Since the unstable network $f_\theta$ also diverges from the correct output, its performance becomes uncorrelated with the target output $y$, resulting in an expected loss similar to that of a random network:
    \[
    \mathbb{E} \left[ \text{Loss}(f_\theta(x), y) \right] \approx \text{Constant}.
    \]

3. \textbf{Convergence of Expected Losses}:

    As the instability of $f_\theta$ increases, the distribution of its output becomes increasingly similar to that of $f_{\text{random}}$. In the limit, the expected loss of $f_\theta$ approaches the expected loss of $f_{\text{random}}$:
    \[
    \lim_{\text{instability} \to \infty} \mathbb{E} \left[ \text{Loss}(f_\theta(x), y) \right] = \mathbb{E} \left[ \text{Loss}(f_{\text{random}}(x), y) \right].
    \]
    This indicates that the unstable network is no more learnable than a random network.

    Thus the unstable RNN $f_\theta$ behaves like a random network in terms of expected loss, it lacks the capacity to learn the correct language $L$. Thus, the learnability of the unstable system is equivalent to that of a random network.
\end{proof}

Next, we provide error bounds for fully-trained, partially-trained, and frozen networks, which is crucial to derive learnability error bounds in practice, as ideally, the error for the stable system should be much lower compared to these variants. Formally, we show
\begin{theorem}[Learnability of Gradient-Descent Trained Systems]
Let $L$ be a formal language recognized by a Pushdown Automaton (PDA) $M = (Q, \Sigma, \Gamma, \delta, q_0, Z_0, F)$. Consider a stack-augmented Recurrent Neural Network (RNN) $f_\theta$ parameterized by $\theta = (\theta_c, \theta_m)$, where $\theta_c$ represents the parameters of the RNN/controller and $\theta_m$ represents the parameters governing the stack/memory operations.

Define the following configurations:
\begin{itemize}
    \item $f_\theta^{\text{full}}$: Both controller $\theta_c$ and memory $\theta_m$ are trainable.
    \item $f_{\theta_c}^{\text{frozen}}$: The controller $\theta_c$ is frozen, and only the memory $\theta_m$ is trainable.
    \item $f_{\theta_m}^{\text{frozen}}$: The memory $\theta_m$ is frozen, and only the controller $\theta_c$ is trainable.
    \item $f_{\theta}^{\text{frozen}}$: Both the controller $\theta_c$ and the memory $\theta_m$ are frozen (i.e., the system is untrained or random).
\end{itemize}

Let the loss function be $\text{Loss}(f_\theta(x), y)$, where $x \in \Sigma^*$ is an input sequence, and $y$ is the target output. The following conditions hold:

1. \textbf{Stable System Minimizes Error}:

    A stable system, where both $\theta_c$ and $\theta_m$ are trained, achieves the lowest expected loss:
    \begin{align*}
    \mathbb{E} \left[ \text{Loss}(f_\theta^{\text{full}}(x), y) \right] &< \min \Bigg\{ \mathbb{E} \left[ \text{Loss}(f_{\theta_c}^{\text{frozen}}(x), y) \right], \\
    &\quad \mathbb{E} \left[ \text{Loss}(f_{\theta_m}^{\text{frozen}}(x), y) \right], \\
    &\quad \mathbb{E} \left[ \text{Loss}(f_{\theta}^{\text{frozen}}(x), y) \right] \Bigg\}.
\end{align*}

    % \[
    % \mathbb{E} \left[ \text{Loss}(f_\theta^{\text{full}}(x), y) \right] < \min \left\{ \mathbb{E} \left[ \text{Loss}(f_{\theta_c}^{\text{frozen}}(x), y) \right], \mathbb{E} \left[ \text{Loss}(f_{\theta_m}^{\text{frozen}}(x), y) \right], \mathbb{E} \left[ \text{Loss}(f_{\theta}^{\text{frozen}}(x), y) \right] \right\}.
    % \]

2. \textbf{Unstable System’s Error is Similar to Untrained System}:

    If the system is unstable, its behavior resembles that of an untrained or random network, leading to an expected error similar to a completely frozen system:
    \[
    \mathbb{E} \left[ \text{Loss}(f_\theta^{\text{unstable}}(x), y) \right] \approx \mathbb{E} \left[ \text{Loss}(f_{\theta}^{\text{frozen}}(x), y) \right].
    \]

3. \textbf{Bounded Error Growth for Stable Systems}:

    A stable system exhibits bounded error growth as sequence length increases. There exists a constant $C > 0$ such that:
    \[
    \left| \text{Loss}(f_\theta^{\text{full}}(x), y) \right| \leq C \quad \forall T \in \mathbb{N}.
    \]

4. \textbf{Consistency Across Sequence Lengths}:

    A stable system exhibits consistent error across sequences of varying lengths:
    \[
    \text{Var}\left( \text{Loss}(f_\theta^{\text{full}}(x), y) \right) < \text{Var}\left( \text{Loss}(f_{\theta_c}^{\text{frozen}}(x), y) \right) \quad \forall T \in \mathbb{N},
    \]
    where $\text{Var}(\cdot)$ represents the variance of the error as a function of sequence length.
\end{theorem}

\begin{proof}

1. \textbf{A Stable System Minimizes Error}:

    In a stable system, gradient descent optimizes both $\theta_c$ and $\theta_m$ to minimize the loss function:
    \[
    \theta^* = \arg\min_{\theta} \mathbb{E} \left[ \text{Loss}(f_\theta(x), y) \right].
    \]
    Since both $\theta_c$ and $\theta_m$ are jointly optimized, the fully trained system $f_\theta^{\text{full}}$ achieves lower loss than any system where one or both components are frozen:
    \begin{align*}
    \mathbb{E} \left[ \text{Loss}(f_\theta^{\text{full}}(x), y) \right] &< \min \Bigg\{ \mathbb{E} \left[ \text{Loss}(f_{\theta_c}^{\text{frozen}}(x), y) \right], \\
    &\quad \mathbb{E} \left[ \text{Loss}(f_{\theta_m}^{\text{frozen}}(x), y) \right], \\
    &\quad \mathbb{E} \left[ \text{Loss}(f_{\theta}^{\text{frozen}}(x), y) \right] \Bigg\}.
\end{align*}

    % \[
    % \mathbb{E} \left[ \text{Loss}(f_\theta^{\text{full}}(x), y) \right] < \min \left\{ \mathbb{E} \left[ \text{Loss}(f_{\theta_c}^{\text{frozen}}(x), y) \right], \mathbb{E} \left[ \text{Loss}(f_{\theta_m}^{\text{frozen}}(x), y) \right], \mathbb{E} \left[ \text{Loss}(f_{\theta}^{\text{frozen}}(x), y) \right] \right\}.
    % \]

2. \textbf{Next we show that Unstable System’s Error is Similar to Untrained System}:

    If the system is unstable, it fails to correctly model the state transitions and stack operations. Let the target output be $M(x)$, where $M$ represents the PDA. For an unstable system, the output diverges from $M(x)$ as the sequence length $T$ increases:
    \[
    \lim_{T \to \infty} \left| f_\theta(x) - M(x) \right| \to \infty.
    \]
    Consequently, the system behaves similarly to a random or untrained network, leading to:
    \[
    \mathbb{E} \left[ \text{Loss}(f_\theta^{\text{unstable}}(x), y) \right] \approx \mathbb{E} \left[ \text{Loss}(f_{\theta}^{\text{frozen}}(x), y) \right].
    \]

3. \textbf{Next we show that Bounded Error Growth for Stable Systems}:

    For a stable system, the optimization process ensures that the error remains bounded across varying sequence lengths. There exists a constant $C > 0$ such that:
    \[
    \left| \text{Loss}(f_\theta^{\text{full}}(x), y) \right| \leq C \quad \forall T \in \mathbb{N}.
    \]
    This condition is satisfied because the model correctly captures the underlying dynamics of the PDA, preventing unbounded error growth.

4. \textbf{Next we show the the Consistency Across Sequence Lengths}:

    In a stable system, both the controller and memory are optimized to handle sequences of varying lengths. This results in low variance in the error:
    \[
    \text{Var}\left( \text{Loss}(f_\theta^{\text{full}}(x), y) \right) < \text{Var}\left( \text{Loss}(f_{\theta_c}^{\text{frozen}}(x), y) \right).
    \]
    The low variance implies that the system is robust to changes in sequence length, a key characteristic of stability.

\end{proof}

Next, we explore the conditions under which a frozen RNN with a trainable stack can outperform an unstable fully trained model. This comparison is significant for understanding how different components (controller versus memory) contribute to overall performance, especially when the model faces complex or recursive patterns inherent in context-free languages.  The analysis shows that freezing the RNN while allowing the stack to adapt can stabilize performance, leading to more reliable error bounds.

\begin{theorem}[Frozen RNN with Trainable Memory Outperforms an Unstable Fully-Trained Model]
Let $L$ be a formal language recognized by a Pushdown Automaton (PDA) $M = (Q, \Sigma, \Gamma, \delta, q_0, Z_0, F)$. Consider a stack-augmented Recurrent Neural Network (RNN) $f_\theta$ parameterized by $\theta = (\theta_c, \theta_m)$, where:
\begin{itemize}
    \item $\theta_c$ represents the parameters of the RNN/controller.
    \item $\theta_m$ represents the parameters governing the stack operations (memory).
\end{itemize}

Define the following configurations:
\begin{enumerate}
    \item $f_{\theta_c}^{\text{frozen}}$: The RNN/controller $\theta_c$ is frozen, and only the memory $\theta_m$ is trainable.
    \item $f_\theta^{\text{unstable}}$: Both $\theta_c$ and $\theta_m$ are trainable, but the model is unstable.
    \item $f_{\theta}^{\text{frozen}}$: Both the RNN/controller $\theta_c$ and memory $\theta_m$ are frozen (i.e., the system is untrained or random).
\end{enumerate}

Let the loss function be $\text{Loss}(f_\theta(x), y)$ for input sequence $x \in \Sigma^*$ and target $y$. The following results hold:

1. \textbf{Partially Frozen Model Can Outperform an Unstable Model}:

    There exists a range $T \in [T_{\text{low}}, T_{\text{high}}]$ such that:
    \[
    \mathbb{E}\left[\text{Loss}(f_{\theta_c}^{\text{frozen}}(x), y)\right] < \mathbb{E}\left[\text{Loss}(f_\theta^{\text{unstable}}(x), y)\right],
    \]
    for $x$ where $|x| \in [T_{\text{low}}, T_{\text{high}}]$.

2. \textbf{Error Growth Bound for Unstable Models}:

    For the fully trained unstable model, the error can exhibit super-linear growth, which can be bounded as:
    \[
    \text{Loss}(f_\theta^{\text{unstable}}(x), y) \leq a \cdot |x|^b + c,
    \]
    where $b > 1$ indicates super-linear growth, and $a, c > 0$ are constants dependent on the degree of instability.

3. \textbf{Performance Bound for Partially Frozen Model}:

    The error for a frozen RNN with trainable memory is bounded by:
    \[
    \text{Loss}(f_{\theta_c}^{\text{frozen}}(x), y) \leq a' \cdot |x| + c',
    \]
    where $a', c' > 0$ are constants, and $a'$ is typically smaller than the corresponding constant $a$ for the unstable model. This linear growth ensures better performance for a stable stack-augmented model even if the controller is frozen.

\end{theorem}
We provide detailed proof in appendix and prove each conditions highlighted above.
We then examine error bounds in unstable models, demonstrating how these models might initially perform well on slightly longer sequences but eventually deteriorate as sequence length increases. This degradation is particularly relevant when assessing models designed to recognize languages with nested structures. The theorems characterize the error growth in unstable systems, providing insight into the rapid error accumulation and eventual convergence to random guessing.

\begin{theorem}[Error Bounds of an Unstable System]
Let $f_\theta$ be a stack-augmented Recurrent Neural Network (RNN) parameterized by $\theta = (\theta_c, \theta_m)$, where:
\begin{itemize}
    \item $\theta_c$ represents the parameters of the RNN/controller,
    \item $\theta_m$ represents the parameters governing the stack operations (memory).
\end{itemize}

The system is considered unstable if it fails to maintain consistent performance as sequence length increases. Let the loss function be $\text{Loss}(f_\theta(x), y)$ for input sequence $x \in \Sigma^*$ and target $y$. Then the following results hold:

1. \textbf{Initial Performance on Slightly Longer Sequences}:

    There exists a range of sequence lengths $|x| \in [T_{\text{train}}, T_{\text{train}} + \Delta T]$, where $\Delta T$ is small, such that the model may exhibit near-perfect performance (e.g., low error):
    \[
    \text{Loss}(f_\theta(x), y) \approx 0 \quad \text{for } |x| \in [T_{\text{train}}, T_{\text{train}} + \Delta T].
    \]

2. \textbf{Error Growth Beyond the Initial Range}:

    As the sequence length continues to increase beyond $T_{\text{train}} + \Delta T$, the error grows rapidly, leading to:
    \[
    \text{Loss}(f_\theta(x), y) \leq a \cdot |x|^b + c \quad \text{for } |x| > T_{\text{train}} + \Delta T,
    \]
    where $a > 0$, $b > 1$, and $c > 0$ indicate super-linear error growth.

3. \textbf{Convergence to Random Guessing or Worse}:

    As sequence length $|x|$ continues to grow, the model’s performance deteriorates further, ultimately converging to random guessing. The expected loss can be bounded as:
    \[
    \lim_{|x| \to \infty} \text{Loss}(f_\theta(x), y) \geq \text{Loss}_{\text{random}} = \frac{1}{k},
    \]
    where $k$ is the number of classes. In some cases, the error may exceed this bound due to the instability:
    \[
    \lim_{|x| \to \infty} \text{Loss}(f_\theta(x), y) \geq \frac{1}{k} + \epsilon \quad \text{for some small } \epsilon > 0.
    \]

\end{theorem}

The detailed proof is in the appendix. In the next section, we conduct a series of experiments showing that our methodology connects these theoretical results to practical considerations.

% Finally, the methodology connects these theoretical results to practical considerations by illustrating the importance of evaluating performance over a wide range of sequence lengths. The analysis is grounded in the idea that consistent behavior across different input lengths is a hallmark of a stable and expressive model. By systematically comparing frozen versus trainable components and analyzing the error dynamics, we gain a comprehensive understanding of when and why instability arises in stack-augmented RNNs.

% The results from this analysis offer a theoretical foundation for designing more robust systems capable of handling formal languages while maintaining stability and generalization across diverse input conditions. The focus is on quantifying and bounding error growth to identify the regions where the model exhibits reliable performance and where it becomes unstable.
\begin{table*}[h!]
    %\centering
    \begin{tabular}{cc}

    % Sub-table 1
    \begin{subtable}[b]{0.45\textwidth}
        %\centering
        \begin{tabular}{|c|c|c|c|}
        \hline
        \textbf{model\_name} & \textbf{bin0} & \textbf{bin1} & \textbf{bin2} \\ \hline
        lstm (n) & \textbf{0.98} & 0.87 & 0.75 \\ \hline
        lstm (c) & 0.96 & \textbf{0.98} & \textbf{0.99} \\ \hline
        jm-hidden (n) & 0.99 & 0.94 & 0.81 \\ \hline
        jm-hidden (m) & 0.99 & \textbf{0.99} & \textbf{0.97} \\ \hline
        rns-3-3 (n) & 0.98 & \textbf{0.94} & \textbf{0.85} \\ \hline
        rns-3-3 (m) & 0.98 & 0.85 & 0.72 \\ \hline
        \end{tabular}
        \caption{count-3}
    \end{subtable} &

    % Sub-table 2
    \begin{subtable}[b]{0.45\textwidth}
        %\centering
        \begin{tabular}{|c|c|c|c|}
        \hline
        \textbf{model\_name} & \textbf{bin0} & \textbf{bin1} & \textbf{bin2} \\ \hline
        lstm (n) & \textbf{0.76} & 0.49 & 0.48 \\ \hline
        lstm (c) & 0.49 & \textbf{0.50} & \textbf{0.50} \\ \hline
        jm-hidden (n) & \textbf{0.79} & 0.48 & 0.48 \\ \hline
        jm-hidden (m) & 0.78 & \textbf{0.56} & \textbf{0.50} \\ \hline
        rns-3-3 (n) & \textbf{0.82} & \textbf{0.75} & \textbf{0.67} \\ \hline
        rns-3-3 (m) & 0.78 & 0.47 & 0.41 \\ \hline
        \end{tabular}
        \caption{marked-reverse-and-copy}
    \end{subtable} \\

    % Sub-table 3
    \begin{subtable}[b]{0.45\textwidth}
        %\centering
        \begin{tabular}{|c|c|c|c|}
        \hline
        \textbf{model\_name} & \textbf{bin0} & \textbf{bin1} & \textbf{bin2} \\ \hline
        lstm (n) & \textbf{0.73} & 0.61 & 0.61 \\ \hline
        lstm (c) & 0.64 & \textbf{0.65} & \textbf{0.66} \\ \hline
        jm-hidden (n) & \textbf{0.80} & \textbf{0.61} & \textbf{0.57} \\ \hline
        jm-hidden (m) & 0.76 & 0.60 & 0.54 \\ \hline
        rns-3-3 (n) & \textbf{0.81} & \textbf{0.69} & \textbf{0.55} \\ \hline
        rns-3-3 (m) & 0.74 & 0.59 & 0.53 \\ \hline
        \end{tabular}
        \caption{count-and-copy}
    \end{subtable} &

    % Sub-table 4
    \begin{subtable}[b]{0.45\textwidth}
        %\centering
        \begin{tabular}{|c|c|c|c|}
        \hline
        \textbf{model\_name} & \textbf{bin0} & \textbf{bin1} & \textbf{bin2} \\ \hline
        lstm (n) & \textbf{0.65} & 0.49 & 0.43 \\ \hline
        lstm (c) & 0.49 & \textbf{0.49} & \textbf{0.50} \\ \hline
        jm-hidden (n) & \textbf{0.68} & 0.42 & 0.43 \\ \hline
        jm-hidden (m) & 0.63 & \textbf{0.47} & \textbf{0.50} \\ \hline
        rns-3-3 (n) & \textbf{0.73} & 0.50 & 0.32 \\ \hline
        rns-3-3 (m) & 0.57 & \textbf{0.38} & \textbf{0.32} \\ \hline
        \end{tabular}
        \caption{marked-copy}
    \end{subtable} \\

    % Sub-table 5
    \begin{subtable}[b]{0.45\textwidth}
        %\centering
        \begin{tabular}{|c|c|c|c|}
        \hline
        \textbf{model\_name} & \textbf{bin0} & \textbf{bin1} & \textbf{bin2} \\ \hline
        lstm (n) & \textbf{0.68} & \textbf{0.56} & \textbf{0.52} \\ \hline
        lstm (c) & 0.49 & 0.49 & 0.50 \\ \hline
        jm-10 (n) & \textbf{0.70} & \textbf{0.53} & \textbf{0.43} \\ \hline
        jm-10 (m) & 0.68 & 0.51 & 0.39 \\ \hline
        rns-3-3 (n) & \textbf{0.69} & 0.52 & 0.41 \\ \hline
        rns-3-3 (m) & 0.65 & \textbf{0.52} & \textbf{0.49} \\ \hline
        \end{tabular}
        \caption{unmarked-copy-different-alphabets}
    \end{subtable} &

    % Sub-table 6
    \begin{subtable}[b]{0.45\textwidth}
        %\centering
        \begin{tabular}{|c|c|c|c|}
        \hline
        \textbf{model\_name} & \textbf{bin0} & \textbf{bin1} & \textbf{bin2} \\ \hline
        lstm (n) & \textbf{0.67} & \textbf{0.54} & \textbf{0.51} \\ \hline
        lstm (c) & 0.51 & 0.50 & 0.50 \\ \hline
        jm-10 (n) & \textbf{0.70} & \textbf{0.56} & \textbf{0.51} \\ \hline
        jm-10 (m) & 0.66 & 0.54 & 0.50 \\ \hline
        rns-3-3 (n) & \textbf{0.69} & 0.53 & 0.50 \\ \hline
        rns-3-3 (m) & 0.69 & \textbf{0.56} & \textbf{0.51} \\ \hline
        \end{tabular}
        \caption{unmarked-reverse-and-copy}
    \end{subtable} \\

    % Sub-table 7
    % \begin{subtable}[b]{0.3\textwidth}
    %     \centering
    %     \begin{tabular}{|c|c|c|c|}
    %     \hline
    %     \textbf{model\_name} & \textbf{bin0} & \textbf{bin1} & \textbf{bin2} \\ \hline
    %     lstm (n) & \textbf{0.60} & \textbf{0.54} & \textbf{0.51} \\ \hline
    %     lstm (c) & 0.49 & 0.50 & 0.50 \\ \hline
    %     jm-hidden (n) & \textbf{0.63} & \textbf{0.55} & \textbf{0.51} \\ \hline
    %     jm-hidden (m) & 0.62 & 0.54 & 0.50 \\ \hline
    %     rns-3-3 (n) & 0.56 & 0.50 & 0.50 \\ \hline
    %     rns-3-3 (m) & \textbf{0.60} & \textbf{0.53} & \textbf{0.50} \\ \hline
    %     \end{tabular}
    %     \caption{unmarked-copy}
    % \end{subtable} \\

    \end{tabular}
    \caption{Test accuracy on the six context-sensitive languages of the best of 10 random restarts for each architecture.Bin0, bin1, and bin2 contain test set in the range $[40-100]$, $[100-200]$ and $[200-400]$, respectively,where $n$ represents models where all parameters are trained; whereas $m$ represents models where only memory is trained and rest all parameters are kept random}
    \label{tab:subtables}
    \vspace{-0.3cm}
\end{table*}

\begin{table}[]
%\footnotesize
    \centering
        \begin{tabular}{|c|c|c|c|}
        \hline
        \textbf{model\_name} & \textbf{bin0} & \textbf{bin1} & \textbf{bin2} \\ \hline
        lstm (n) & \textbf{0.60} & \textbf{0.54} & \textbf{0.51} \\ \hline
        lstm (c) & 0.49 & 0.50 & 0.50 \\ \hline
        jm-hidden (n) & \textbf{0.63} & \textbf{0.55} & \textbf{0.51} \\ \hline
        jm-hidden (m) & 0.62 & 0.54 & 0.50 \\ \hline
        rns-3-3 (n) & 0.56 & 0.50 & 0.50 \\ \hline
        rns-3-3 (m) & \textbf{0.60} & \textbf{0.53} & \textbf{0.50} \\ \hline
        \end{tabular}
        \caption{Test accuracy on the unmarked-copy CFL of the best of 10 random restarts for each architecture over bin0, bin1, and bin2 respectively}
        \label{tab:unmarked-copy}
\end{table}
\section{Experimental Setup}

We evaluate various RNN architectures on non-context-free languages (non-CFLs) and natural language modeling tasks. For the Penn Treebank (PTB) dataset, we test five models: a standard LSTM \cite{lstmcfg}, two stack-augmented RNNs \cite{joulin2015inferring}, and two nondeterministic stack-augmented RNNs \cite{dusell2023the} (see appendix for details). For non-CFL tasks, we also evaluate one LSTM \cite{lstmcfg}, three stack-augmented RNNs \cite{joulin2015inferring}, and one nondeterministic stack-augmented RNN \cite{dusell2023the}.

% In this study, we evaluate the performance of various Recurrent Neural Network (RNN) architectures on a diverse set of tasks involving non-context-free languages (non-CFLs) and natural language modeling. Specifically, for the Penn Treebank (PTB) dataset, we test five models: a standard LSTM \cite{lstmcfg}, two variations of stack-augmented RNNs \cite{joulin2015inferring}, and two variations of nondeterministic stack-augmented RNNs \cite{dusell2023the} (further details on these models can be found in the appendix). For the non-CFL tasks, we also evaluate five models, including one LSTM \cite{lstmcfg}, three variations of stack-augmented RNNs \cite{joulin2015inferring}, and one nondeterministic stack-augmented RNN \cite{dusell2023the}. \\

\textbf{Non-Context-Free Languages} We investigate how stack-augmented RNNs handle non-CFL phenomena by evaluating them on several complex language modeling tasks. Each of the non-CFLs in our study can be recognized by a real-time three-stack automaton. We examine the following seven tasks:

\begin{enumerate}
\item \textbf{Count-3}: The language {$a^n b^n c^n \mid n \geq 0 $}, which involves counting, reversing, and copying strings with markers.
% \item \textbf{Marked-Reverse-and-Copy}: The language {$w\#w^R\#w \mid w \in {0, 1}^$}, which requires reversing and copying a string with explicit markers.
\item \textbf{Marked-Reverse-and-Copy}: The language $\{ w\#w^R\#w \mid w \in \{0, 1\}^* \}$, which requires reversing and copying a string with explicit markers.

\item \textbf{Count-and-Copy}: The language $\{ w\#^n w \mid w \in {0, 1}^ \}$, which requires counting and copying strings separated by marked divisions.
\item \textbf{Marked-Copy}: The language $\{ w\#w \mid w \in {0, 1}^* \}$, which involves a simple marked copying operation.
\item \textbf{Unmarked-Copy-Different-Alphabets}: The language $\{ ww' \mid w \in {0, 1}^, w' = \phi(w) \}$, where $\phi$ is a homomorphism defined as $\phi(0) = 2, \phi(1) = 3$.
\item \textbf{Unmarked-Reverse-and-Copy}: The language $\{ w w^R w \mid w \in {0, 1}^ \}$, which involves reversing and copying without explicit markers.
\item \textbf{Unmarked-Copy}: The language $\{ ww \mid w \in {0, 1}^* \}$, which requires unmarked copying of sequences.
\end{enumerate}

For consistency, we followed the experimental framework and hyperparameters set by prior work \cite{dusell2023the}. The training and validation sets were sampled from sequences in the length range of $[40, 80]$, while the test sets were split into three bins based on sequence length: bin 0 ($[40, 100]$), bin 1 ($[100, 200]$), and bin 2 ($[200, 400]$). \\
\textbf{Natural Language Modeling} In addition to non-CFL tasks, we evaluated the stack-augmented RNNs on natural language modeling using the Penn Treebank (PTB) dataset, preprocessed as in Mikolov et al. (2011). The hyperparameters used for these experiments were consistent with those from prior work \cite{dusell2023the}.

\begin{table}[h!]
%\footnotesize
%\centering
\begin{tabular}{|c|c|c|c|c|}
\hline
\textbf{Model Name} & \textbf{Test PPL (n)} & \textbf{Test PPL (m)} \\ \hline
lstm-256 & \textbf{119.8} & 129.8 \\ \hline
jm-hidden-247 & 126.8 & 125.1 \\ \hline
jm-learned-22 & 125.9 & 124.2 \\ \hline
rns-4-5 & 126.5 & \textbf{120.5} \\ \hline   
vrns-3-3-5 & 123.5 & 126.0 \\ \hline
\end{tabular}
\caption{Test and validation perplexity on the Penn Treebank of the best of 10 random restarts for each architecture, where $n$ represents models with all trainable parameters, whereas $m$ represents models where only memory is trained.}
\label{tab:ppl_results}
\vspace{-0.5cm}
\end{table}

\section{Discussion and Conclusion}

Our experiments on seven CGL benchmarks highlight the critical role of stability in understanding the learnability of stack-augmented RNNs. As shown in Table \ref{tab:subtables} and \ref{tab:unmarked-copy}, both fully trained models and those with only memory trained perform comparably on short sequences. However, when tested on longer sequences (bin2), most models converge to random guessing, underscoring the instability introduced by increased sequence complexity.
For instance, in the \textit{count-3} task, fully trained models experience a significant accuracy drop—from 98\% to 99\% down to 75\% to 85\%. In contrast, freezing the controller during training allows models like LSTM and jm-hidden to maintain around 96\% accuracy, while fixing both the controller and memory enables four out of five models to preserve their initial performance. This demonstrates that excessive parameter tuning can destabilize learning rather than enhance it. A similar trend is observed in the \textit{count-and-copy} task, where models with fixed components maintain stable accuracy near their initial 64\%, while fully trained models suffer a drop from 73\% to 81\% down to 55\% to 61\%. The PTB dataset also shows that models with only memory trained outperform fully trained counterparts, reinforcing that over-parameterization during training introduces noise and instability (Table \ref{tab:ppl_results}). These results highlight that stability is essential for effective learning. Models exhibiting stable performance across varying sequence lengths are more likely to generalize well, while instability often leads to poor long-term retention and performance degradation. Additional analysis are shown in appendix. 

\section{Conclusion.} Our theoretical analysis demonstrates that unstable systems tend to converge to random guessing, or even worse, under certain conditions. We validated these findings through a series of experiments that consistently showed this behavior in practice. These results highlight that understanding stability and rigorously testing models on longer sequences are crucial for achieving better learnability. In scenarios where longer sequences are not available, the error bounds of the system ideally should be better than random guessing or models with partially frozen parameters. Our study emphasizes that selectively freezing components during training can lead to more stable behavior and improved generalization, especially for longer sequences but is still farway compared to stable models. Thus evaluating stability across varying sequence lengths provides key insights into a model’s ability to maintain robust performance, making stability assessments a critical factor in building more reliable and learnable sequence models.

\bibliography{main}
\bibliographystyle{ieeetr}
%\newpage
%\newpage

\appendix

\section{Appendix A: Related Work}
The capability of Recurrent Neural Networks (RNNs), Gated Recurrent Units (GRUs), and Long Short-Term Memory networks (LSTMs) to learn formal languages has been well-studied. While these architectures perform well on simple languages, such as basic counting tasks and Dyck languages \cite{holldobler1997designing, steijvers2019recurrent, skachkova2018closing}, they struggle with more complex languages, especially when generalizing to longer sequences \cite{sennhauser2018evaluating}. The lack of external memory structures limits these models' ability to handle advanced languages, confining their generalization to sequence lengths close to those seen during training \cite{boden2000context, gers2001lstm}. While some recent work has explored length generalization in synthetic reasoning tasks with large pretrained models \cite{anil2022exploring, dave2024investigating2}, these studies are often disconnected from the more rigorous demands of formal language learning. On other hand several studies have shown that the choice of objective functions \cite{lan2022minimum}  and learning algorithms \cite{mali2021investigating} significantly affects RNNs' ability to stably learn complex grammars. For instance, \cite{lan2022minimum} demonstrated that specialized loss functions, such as minimum description length, lead to more stable convergence and better generalization on formal language tasks.

Although RNNs are theoretically Turing complete, practical constraints such as finite precision and limited recurrent steps diminish their real-world expressivity \cite{chen2017recurrent, perez2019turing}. Under such constraints, standard RNNs and GRUs are restricted to regular languages, while LSTMs can approximate context-free languages by simulating k-counter mechanisms \cite{ackerman2020survey, bhattamishra2020ability, merrill2020formal}. Tensor RNNs \cite{mali2023computational, mali23092tensor} have extended these capabilities, but their higher computational costs make them challenging to deploy.

Memory-augmented architectures have been proposed to overcome these limitations by integrating external structures such as stacks \cite{stogin2024provably, joulin2015inferring, grefenstette2015learning}, random access memory \cite{danihelka2016associative, kurach2015neural}, and memory matrices \cite{graves2014neural, graves2016hybrid}. These models have shown success in recognizing more complex languages and executing tasks like string copying and reversal. However, most research has focused on theoretical expressivity and evaluations on short sequences. Consequently, many studies fail to test these models on longer sequences, where stability issues become more apparent \cite{dave2024investigating2}. Empirical studies have begun to emphasize the need for testing on longer sequences \cite{mali2020neural, mali2021recognizing, mali23092tensor}, and recent theoretical work by Stogin \cite{stogin2024provably} has identified stable differentiable memory structures. Yet, these insights are not comprehensive and lack a systematic analysis of stability.

The primary gap in existing research is the absence of focused investigations into the stability of memory-augmented models, particularly when processing sequences significantly longer than those seen during training. Stability, in this context, refers to a model's ability to maintain consistent performance across varying sequence lengths without degradation in accuracy or memory manipulation. Most studies evaluate models on sequences similar in length to the training data, thereby overlooking the challenges posed by longer inputs. As a result, while these models often achieve high accuracy within the training range, they frequently exhibit instability when exposed to longer sequences.

This study addresses these gaps by examining the stability of stack-augmented RNNs across varying sequence lengths. We explore how different configurations, such as freezing the RNN controller while keeping the memory trainable, impact stability. We derive theoretical conditions for stable performance and identify when models become unstable. Our work is among the first to systematically investigate the stability of memory-augmented neural networks from both theoretical and empirical perspectives, providing insights into factors that influence learnability and generalization across longer and more complex sequences.

\section{Appendix B: Additional Results}

Figure 1-3 and Table 4-7 provide comprehensive overview of our hypothesis. 
In our experiments with the non-context-free language task \textit{count-3}, a notable divergence in performance emerged between fully trained models and those with selective components frozen. During evaluation on sequences within the training range (40 to 100 tokens), all models performed comparably. However, as sequence lengths extended beyond this range (100-200 and 200-400 tokens), fully trained models exhibited significant performance degradation, with accuracy decreasing to between 75\% and 85\%. This decline illustrates the difficulty these models face in generalizing to longer sequences, where learned dependencies may not scale effectively.

Conversely, models with a frozen controller demonstrated remarkable stability, maintaining near-perfect accuracy (up to 99\%) even on the longest test sequences. This suggests that freezing the controller helps mitigate instability introduced by continuous parameter updates, thereby improving the model’s ability to generalize over extended input lengths.

Further analysis of the model \textit{jm-hidden}, which leverages a superposition stack RNN and offloads the controller’s hidden state to external memory, reinforces these observations. When fully trained (\textit{jm-hidden none}), the model’s predictions begin to diverge from the expected outputs as sequence lengths increase, particularly beyond 200 tokens, highlighting instability in handling extended input sequences. However, when the memory component is frozen (\textit{jm-hidden m}), the model exhibits stable performance, closely aligning with the expected outputs across all tested sequence lengths, including the longest sequences.

This stability trend was consistently observed in our experiments on the Penn Treebank (PTB) dataset for language modeling. Models such as \textit{jm-hidden-247}, \textit{jm-learned-22}, and \textit{rns-4-5} with frozen memory components not only retained performance but, in some cases, outperformed their fully trained counterparts, achieving superior perplexity scores (126.8 $\rightarrow$ 125.1, 125.9 $\rightarrow$ 124.2, 126.5 $\rightarrow$ 120.5, respectively). These results underscore the critical role of stability in enhancing the generalization capabilities of neural architectures, particularly in tasks involving long sequences. Freezing specific components, such as memory or the controller, leads to more reliable and robust performance, highlighting the necessity of incorporating stability considerations into neural network design.

%%% Shrabon Please add all results
% \begin{figure*}[htbp]
%     \centering
%     \begin{subfigure}[b]{0.3\textwidth}
%         \centering
%         \includegraphics[width=\textwidth]{example-image-a}
%         \caption{Subfigure 1}
%         \label{fig:subfig1}
%     \end{subfigure}
%     \begin{subfigure}[b]{0.3\textwidth}
%         \centering
%         \includegraphics[width=\textwidth]{example-image-b}
%         \caption{Subfigure 2}
%         \label{fig:subfig2}
%     \end{subfigure}
%     \begin{subfigure}[b]{0.3\textwidth}
%         \centering
%         \includegraphics[width=\textwidth]{example-image-c}
%         \caption{Subfigure 3}
%         \label{fig:subfig3}
%     \end{subfigure}
    
%     \begin{subfigure}[b]{0.3\textwidth}
%         \centering
%         \includegraphics[width=\textwidth]{example-image-a}
%         \caption{Subfigure 4}
%         \label{fig:subfig4}
%     \end{subfigure}
%     \begin{subfigure}[b]{0.3\textwidth}
%         \centering
%         \includegraphics[width=\textwidth]{example-image-b}
%         \caption{Subfigure 5}
%         \label{fig:subfig5}
%     \end{subfigure}
%     \begin{subfigure}[b]{0.3\textwidth}
%         \centering
%         \includegraphics[width=\textwidth]{example-image-c}
%         \caption{Subfigure 6}
%         \label{fig:subfig6}
%     \end{subfigure}
    
%     \begin{subfigure}[b]{0.3\textwidth}
%         \centering
%         \includegraphics[width=\textwidth]{example-image-a}
%         \caption{Subfigure 7}
%         \label{fig:subfig7}
%     \end{subfigure}
    
%     \caption{Main figure with 7 subfigures.}
%     \label{fig:mainfigure}
% \end{figure*}

\begin{figure}[h]
    \centering
    \includegraphics[width=0.9\textwidth]{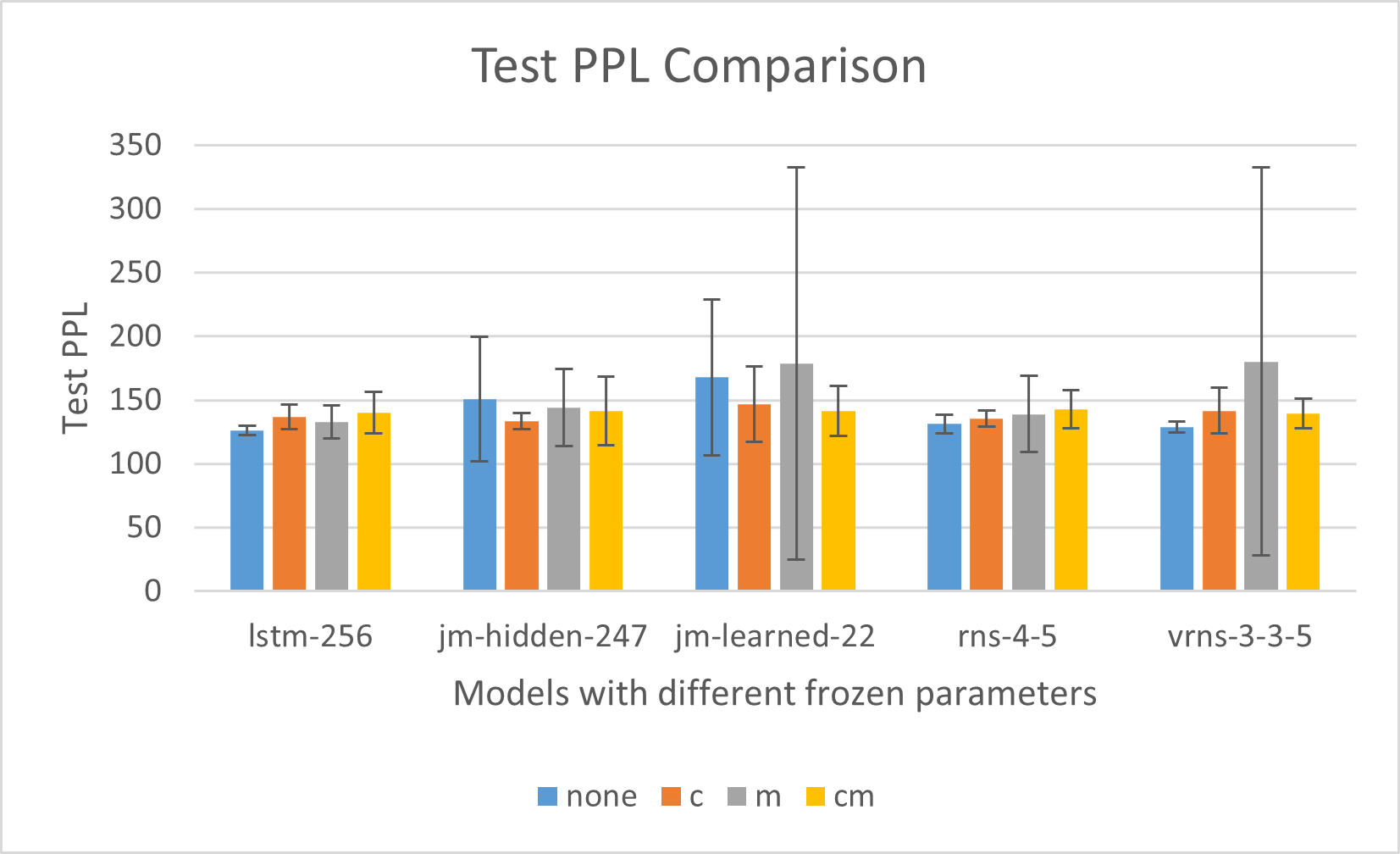}
    \caption{Performance of various models using 4 configuration (none = fully trained model, m = only memory is trained, c = only controller is trained and cm = controller and memory are frozen and only classifier is trainable. We report performance on language modeling task and report perplexity (PPL) on Penn tree bank (PTB) dataset. }
    \label{fig:myimage}
\end{figure}

\begin{figure*}[htbp]
    %\centering
    \begin{subfigure}[b]{0.45\textwidth}
        %\centering
        \includegraphics[width=\textwidth]{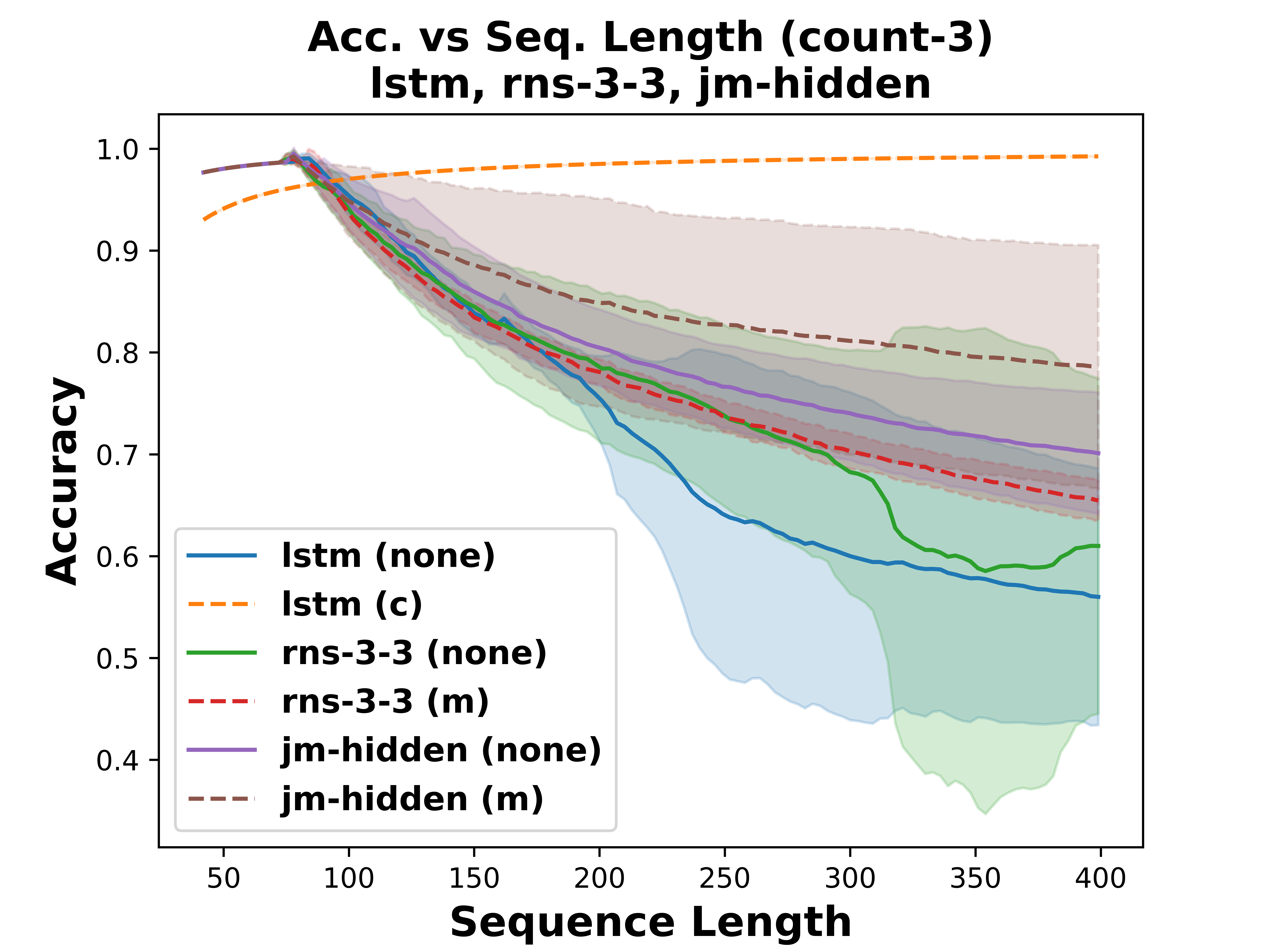}
        \caption{count-3}
        \label{fig:subfig1}
    \end{subfigure}
    \begin{subfigure}[b]{0.45\textwidth}
        %\centering
        \includegraphics[width=\textwidth]{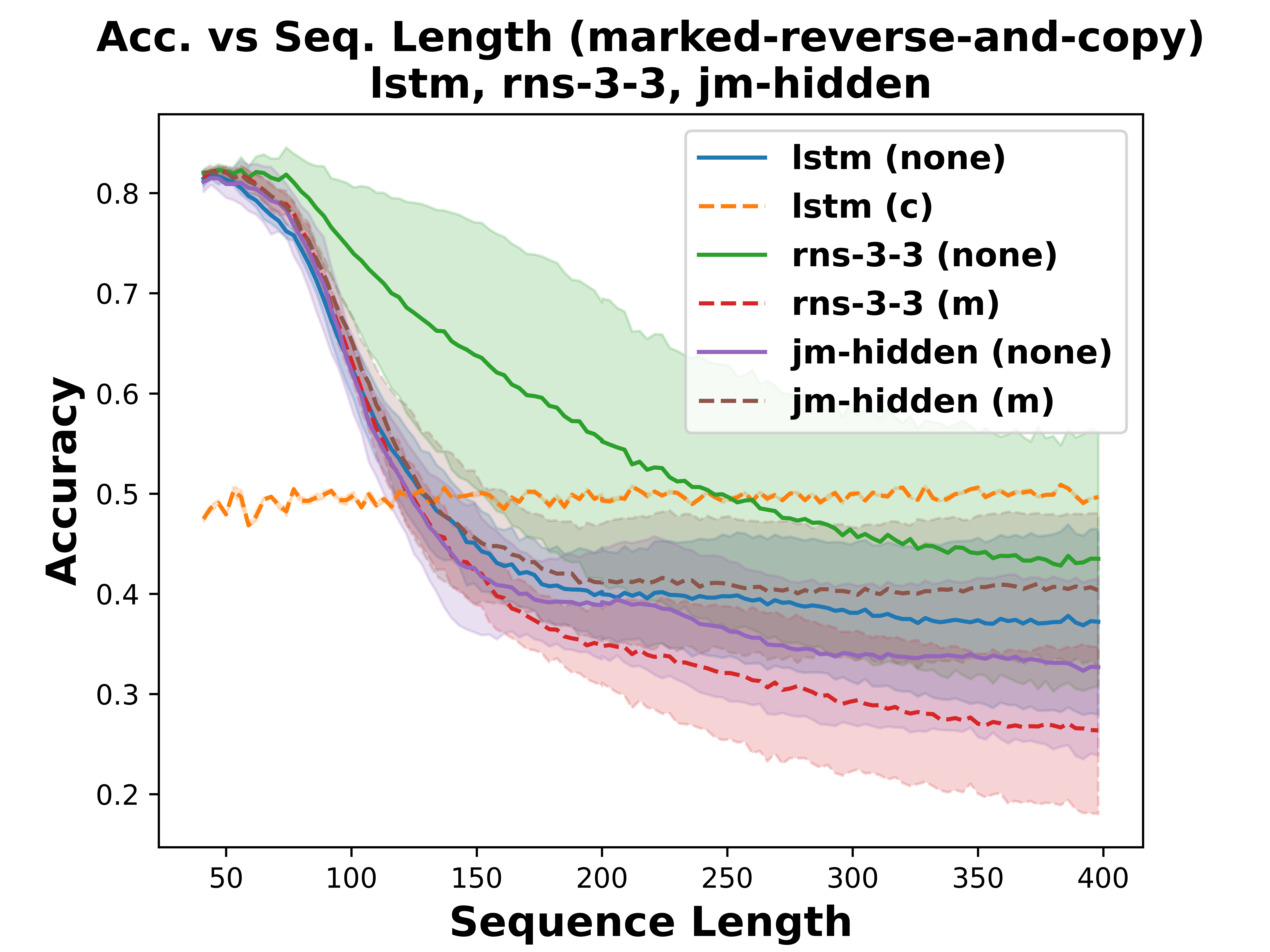}
        \caption{marked-reverse-and-copy}
        \label{fig:subfig2}
    \end{subfigure}
    \begin{subfigure}[b]{0.45\textwidth}
        %\centering
        \includegraphics[width=\textwidth]{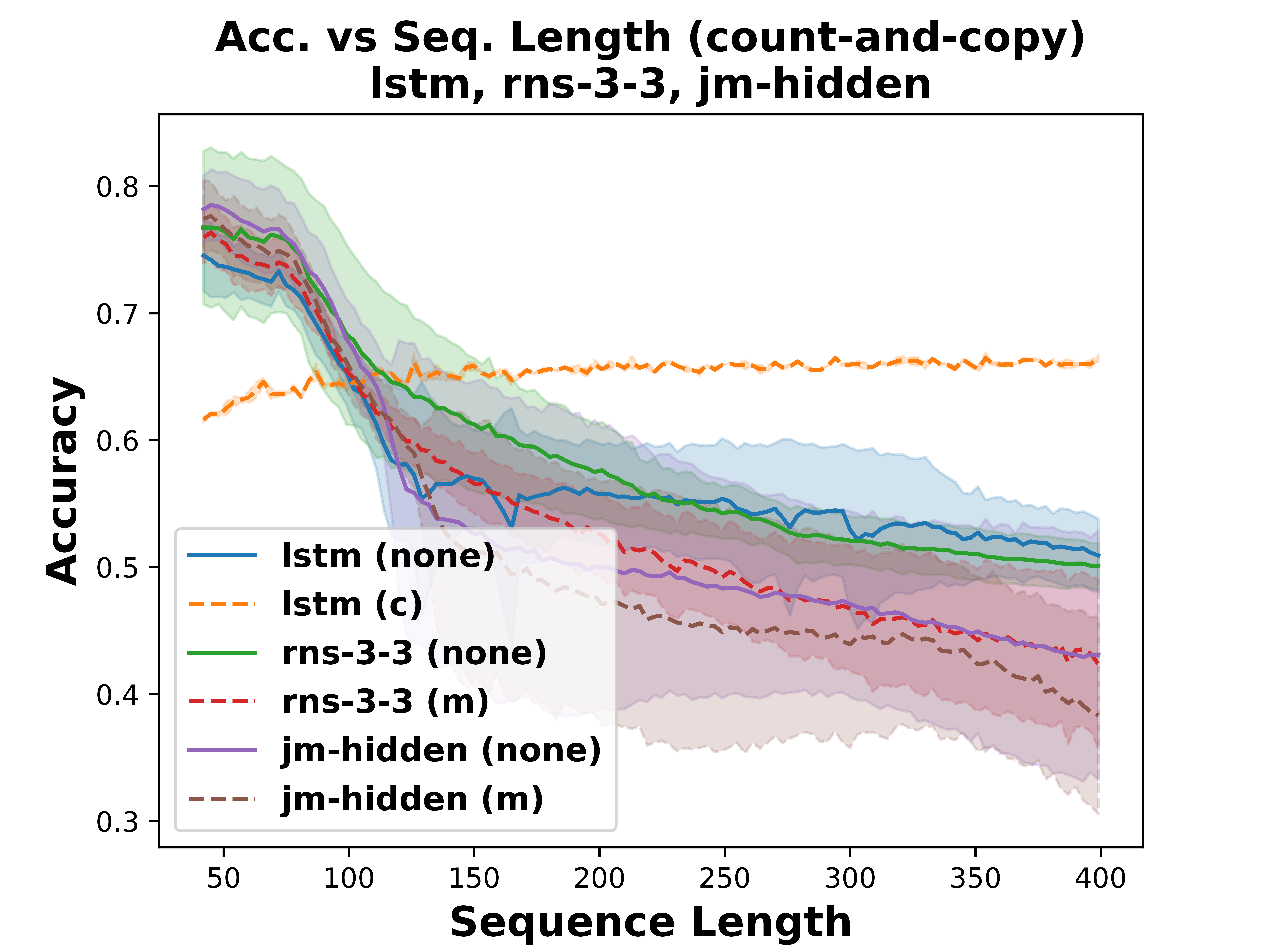}
        \caption{count-and-copy}
        \label{fig:subfig3}
    \end{subfigure}
    \begin{subfigure}[b]{0.45\textwidth}
        %\centering
        \includegraphics[width=\textwidth]{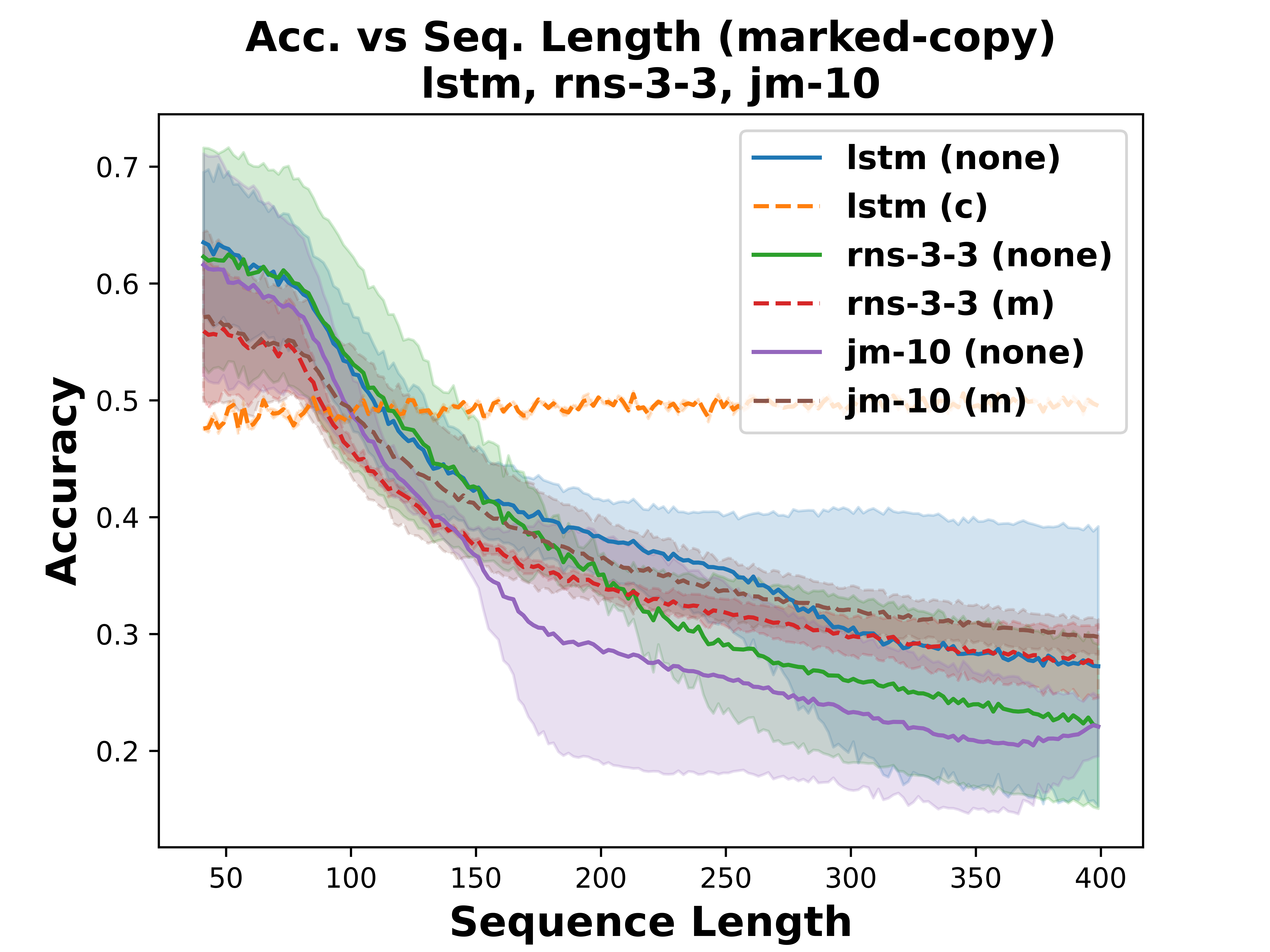}
        \caption{marked-copy}
        \label{fig:subfig4}
    \end{subfigure}
    \begin{subfigure}[b]{0.45\textwidth}
        %\centering
        \includegraphics[width=\textwidth]{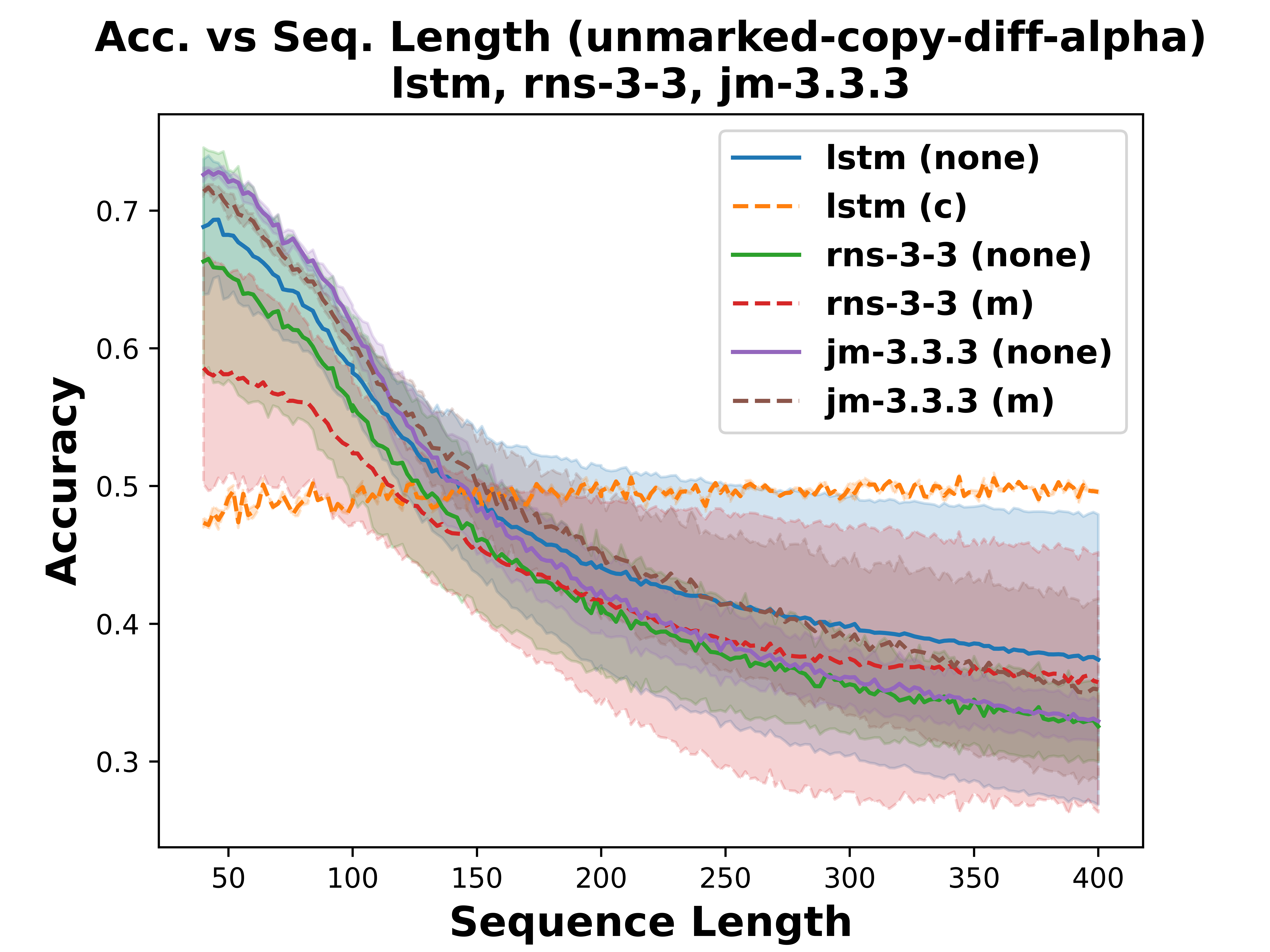}
        \caption{unmarked-copy-different-alphabets}
        \label{fig:subfig5}
    \end{subfigure}
    \begin{subfigure}[b]{0.45\textwidth}
        %\centering
        \includegraphics[width=\textwidth]{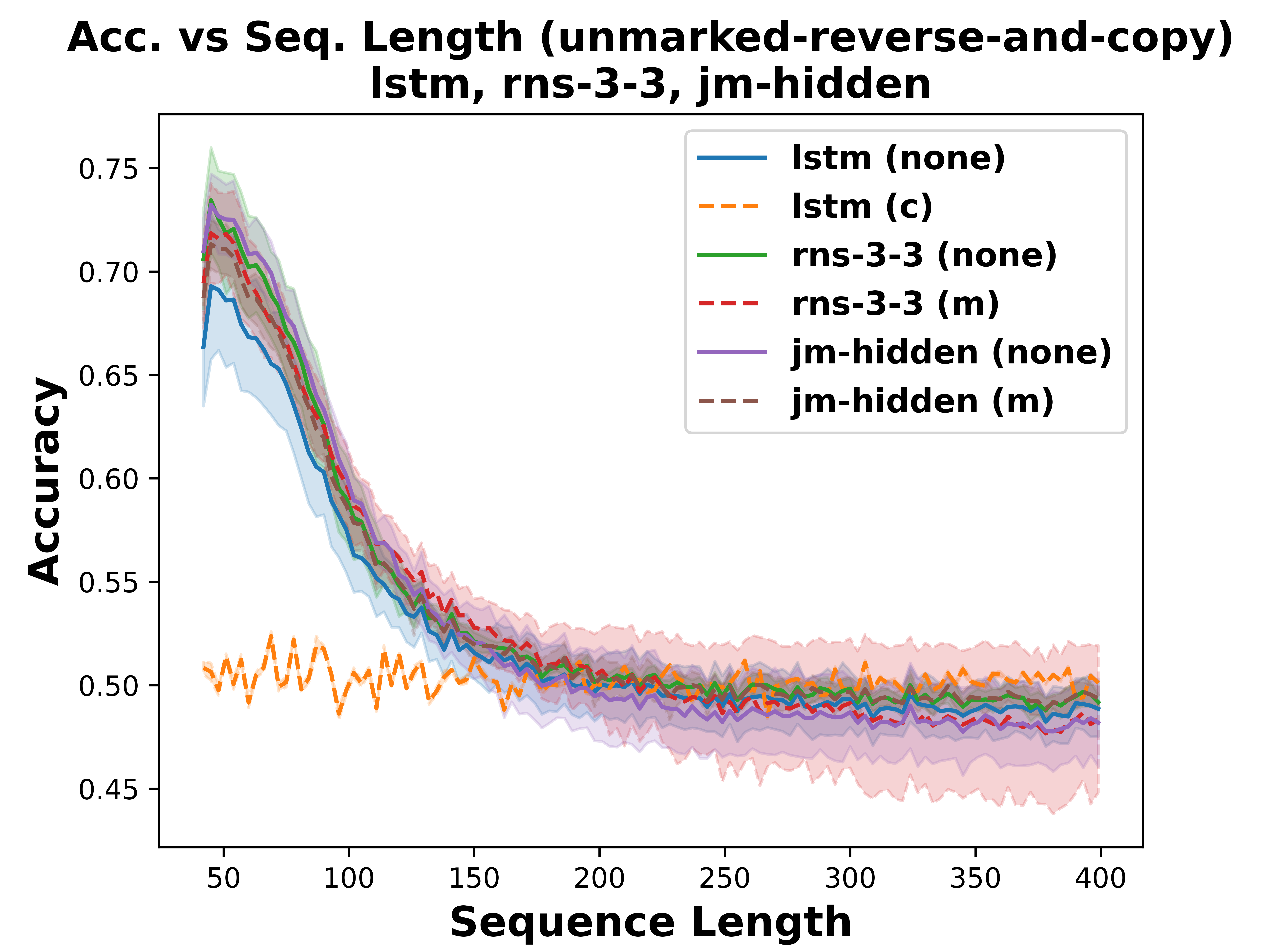}
        \caption{unmarked-reverse-and-copy}
        \label{fig:subfig6}
    \end{subfigure}
    \begin{subfigure}[b]{0.45\textwidth}
        \centering
        \includegraphics[width=\textwidth]{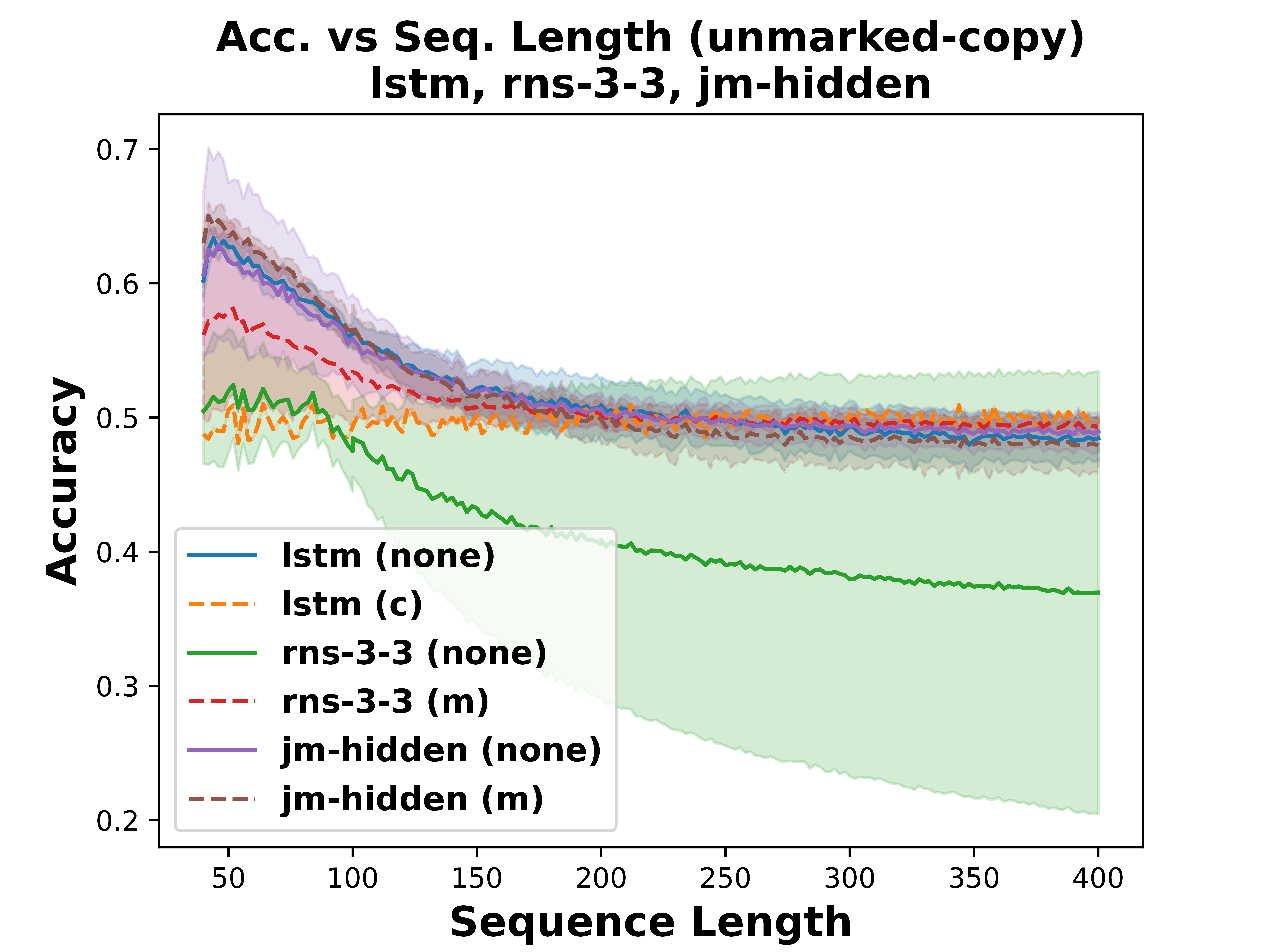}
        \caption{unmarked-copy}
        \label{fig:subfig7}
    \end{subfigure}
    
    \caption{Performance of top 3 models on test sets across 7 context free languages, when models are fully trained (none) and when only memory (m) is frozen. }
    \label{fig:mainfigure}
\end{figure*}

\begin{figure*}[htbp]
    %\centering
    \begin{subfigure}[b]{0.45\textwidth}
        %\centering
        \includegraphics[width=\textwidth]{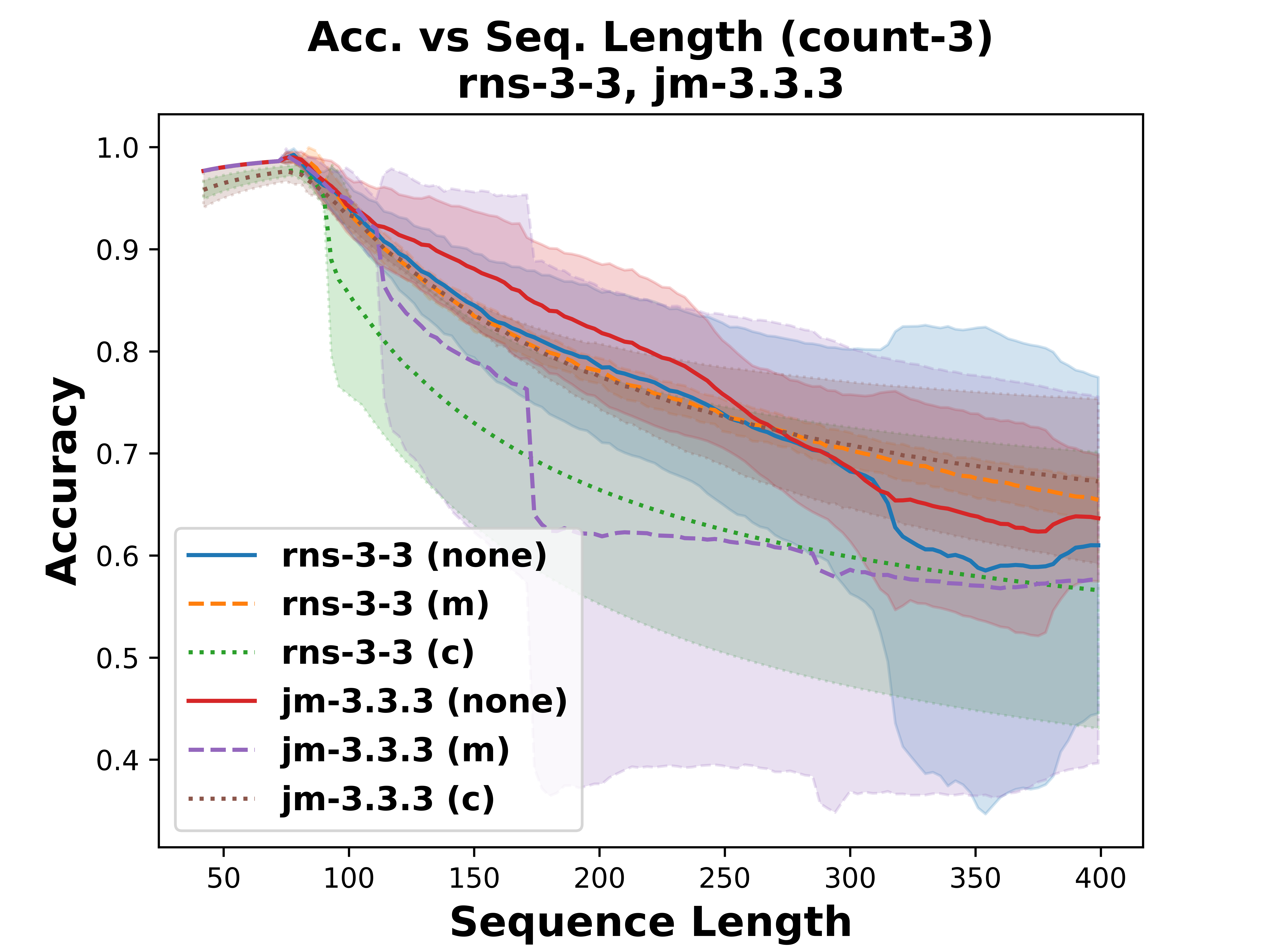}
        \caption{count-3}
        \label{fig:subfig1}
    \end{subfigure}
    \begin{subfigure}[b]{0.45\textwidth}
        %\centering
        \includegraphics[width=\textwidth]{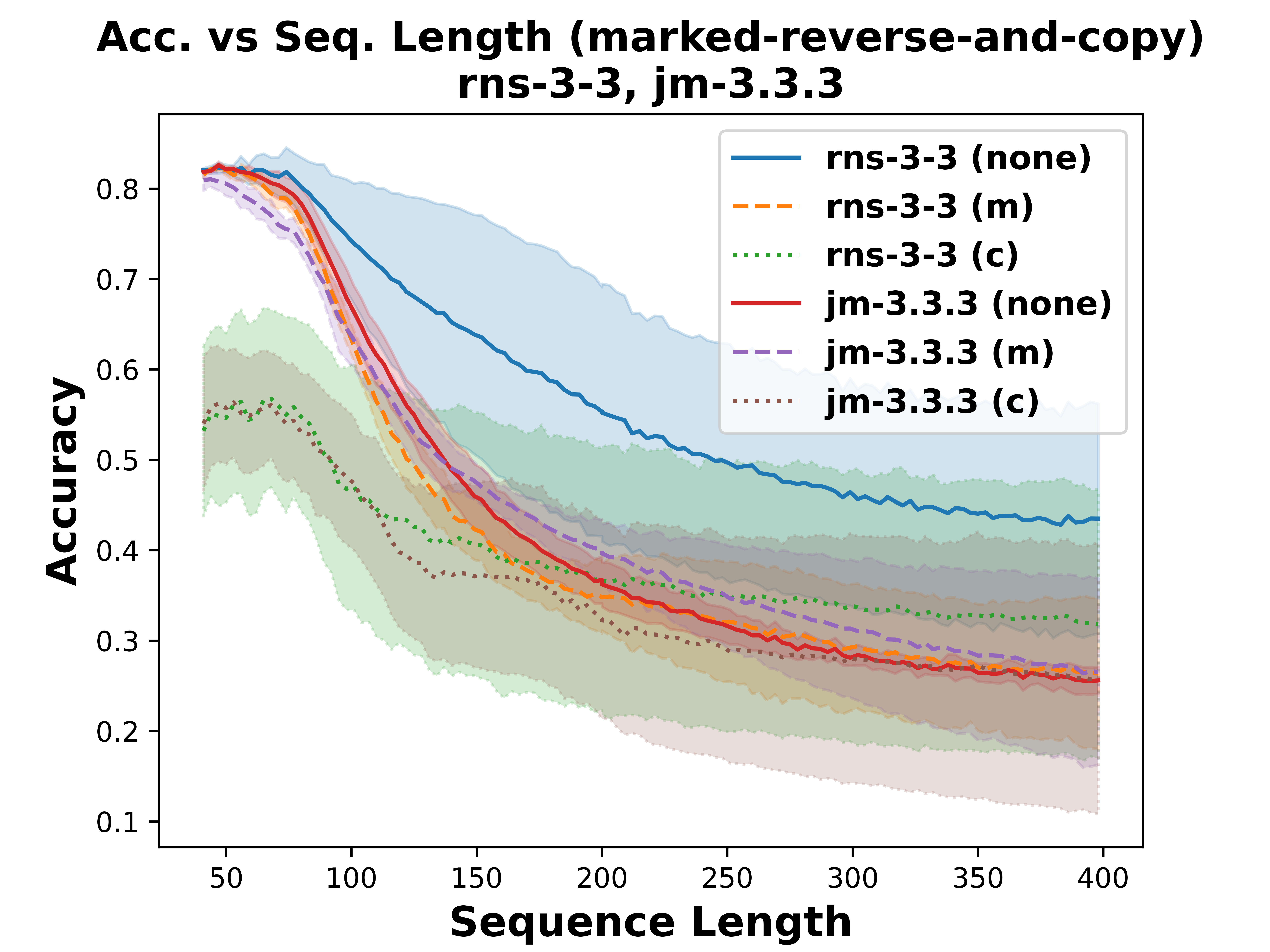}
        \caption{marked-reverse-and-copy}
        \label{fig:subfig2}
    \end{subfigure}
    \begin{subfigure}[b]{0.45\textwidth}
        %\centering
        \includegraphics[width=\textwidth]{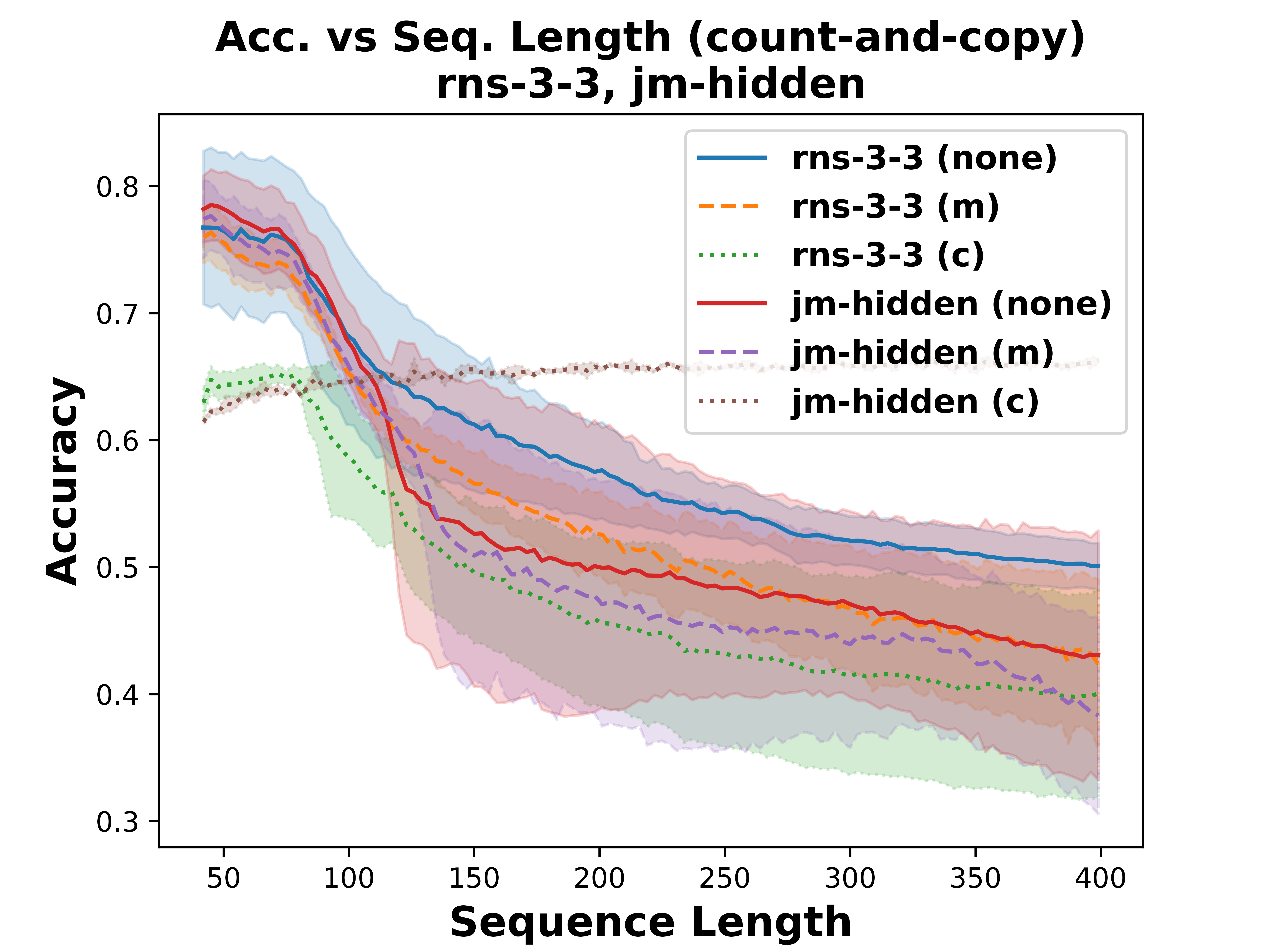}
        \caption{count-and-copy}
        \label{fig:subfig3}
    \end{subfigure}
    \begin{subfigure}[b]{0.45\textwidth}
        %\centering
        \includegraphics[width=\textwidth]{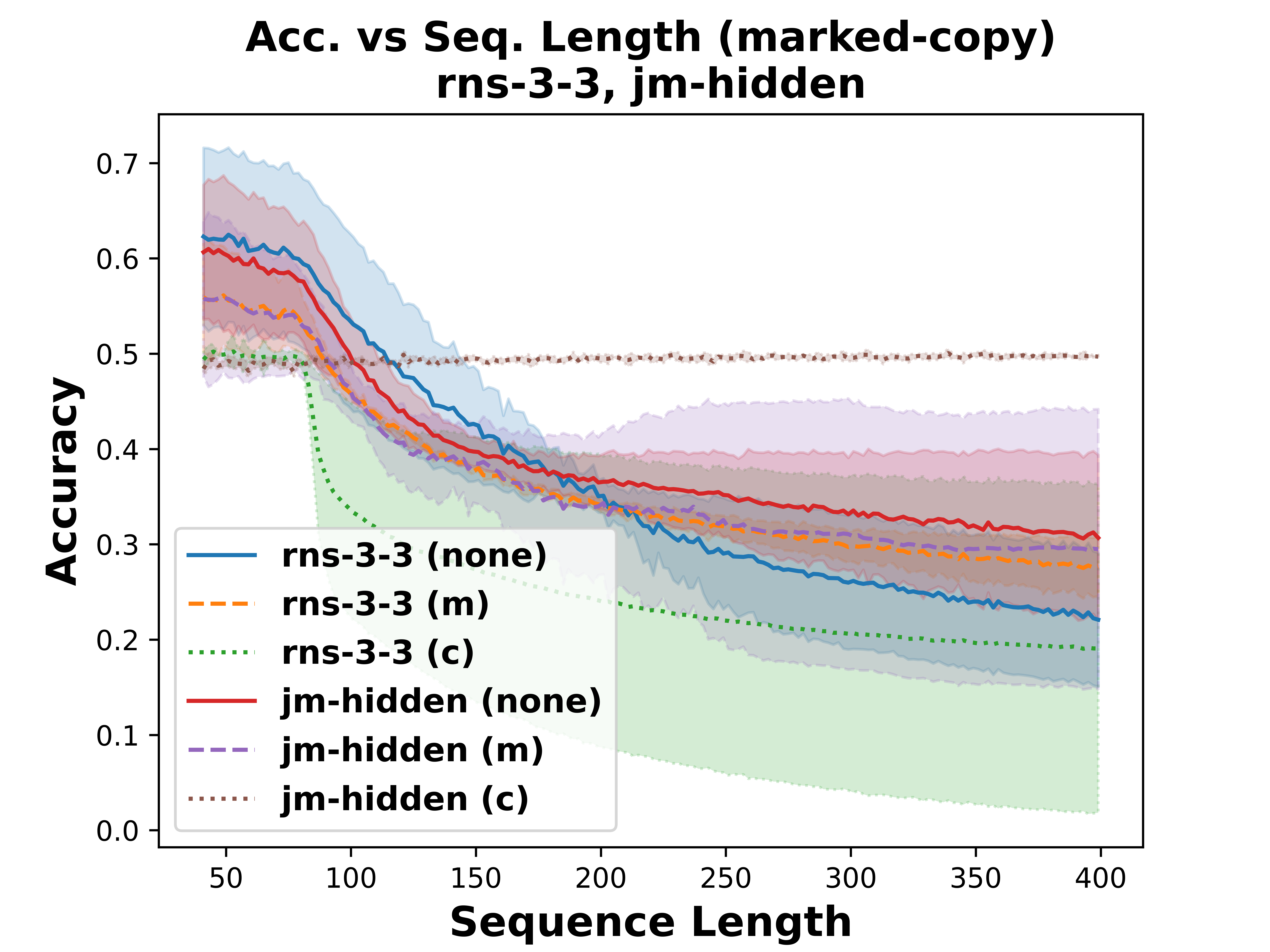}
        \caption{marked-copy}
        \label{fig:subfig4}
    \end{subfigure}
    \begin{subfigure}[b]{0.45\textwidth}
        %\centering
        \includegraphics[width=\textwidth]{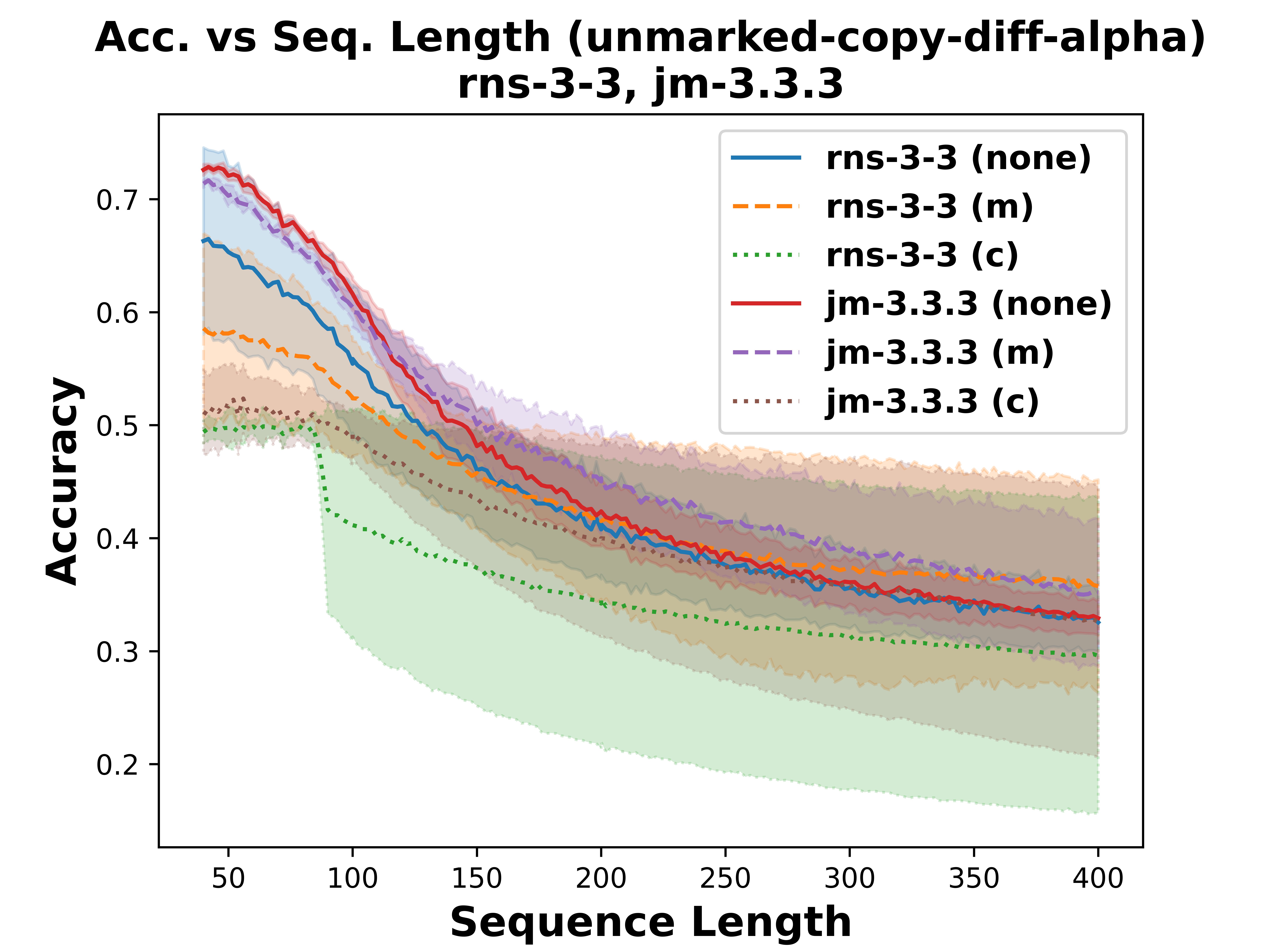}
        \caption{unmarked-copy-different-alphabets}
        \label{fig:subfig5}
    \end{subfigure}
    \begin{subfigure}[b]{0.45\textwidth}
        %\centering
        \includegraphics[width=\textwidth]{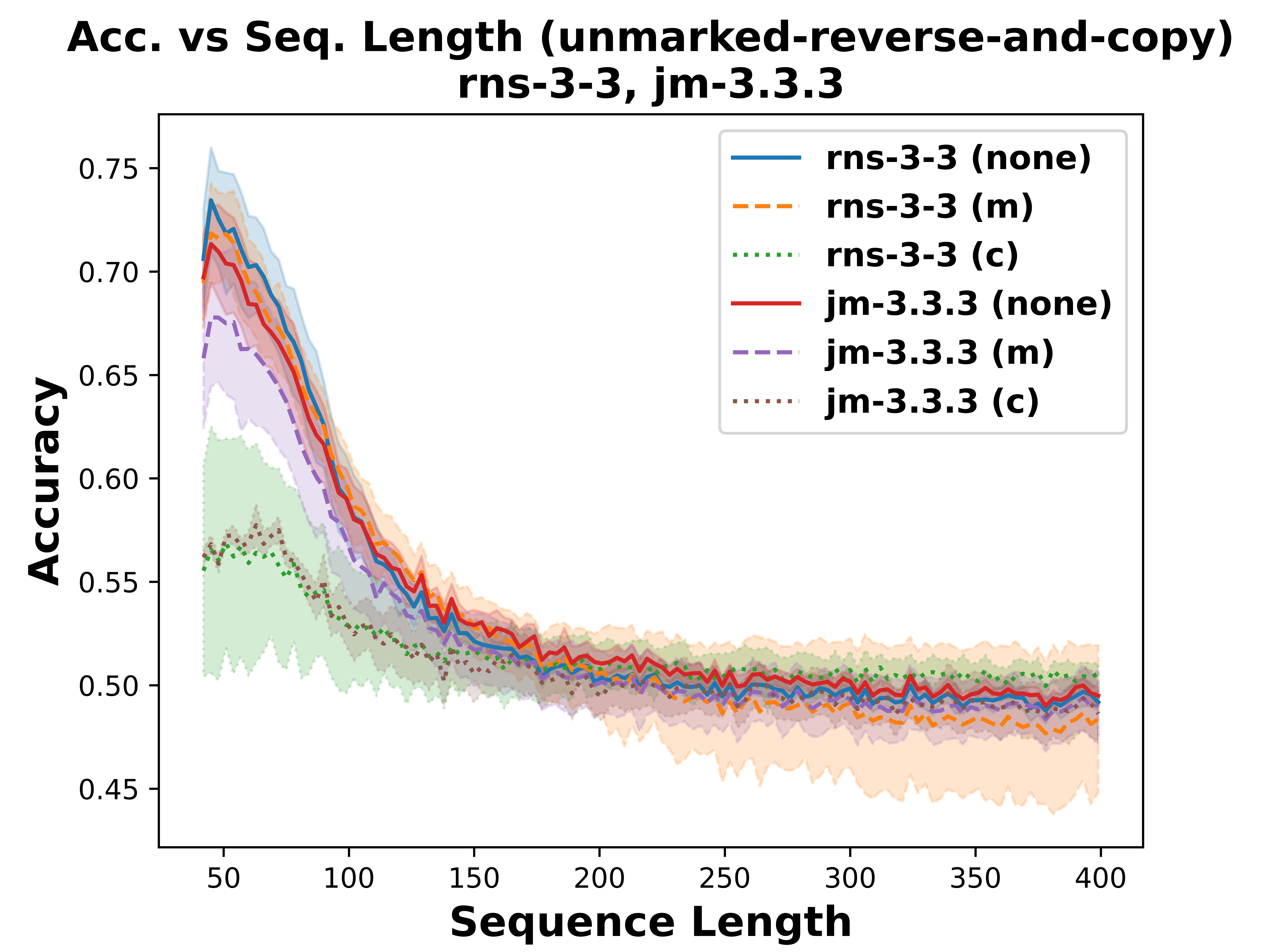}
        \caption{unmarked-reverse-and-copy}
        \label{fig:subfig6}
    \end{subfigure}

    \begin{subfigure}[b]{0.45\textwidth}
        \centering
        \includegraphics[width=\textwidth]{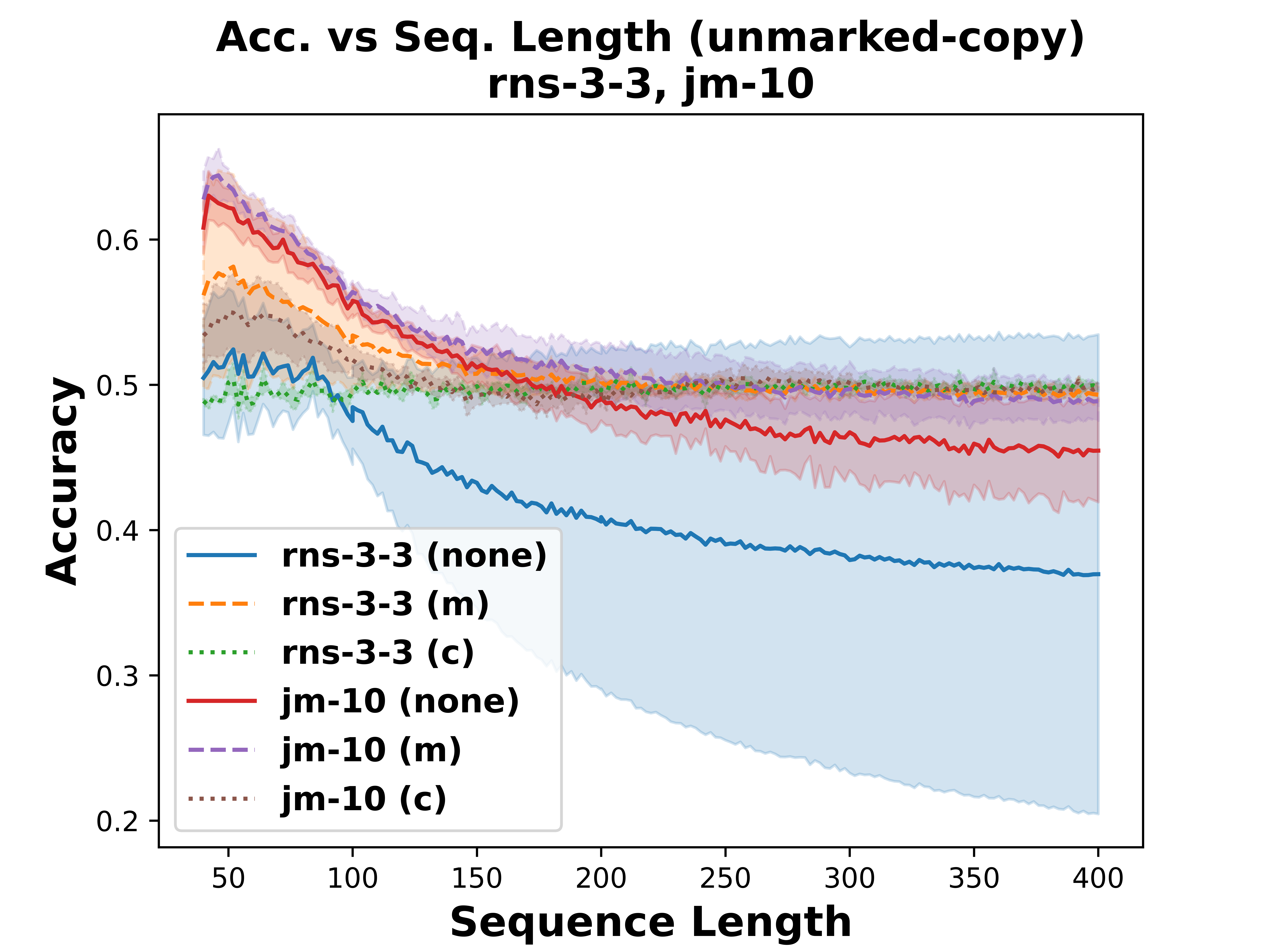}
        \caption{unmarked-copy}
        \label{fig:subfig7}
    \end{subfigure}
    
    \caption{Performance of top 2 memory-augmented models on test sets across 7 context free languages, when models are fully trained (none) and when only memory (m) is frozen.}
    \label{fig:mainfigure}
\end{figure*}

%%% Table

        % \multirow{5}{*}{\textbf{count-and-copy}} & lstm &  &  &  &  &  &  &  \\ \cline{2-9} 
        % & jm-10 &  &  &  &  &  &  &  \\ \cline{2-9} 
        % & jm-3.3.3 &  &  &  &  &  &  &  \\ \cline{2-9} 
        % & jm-hidden &  &  &  &  &  &  &  \\ \cline{2-9} 
        % & rns-3-3 &  &  &  &  &  &  &  \\ \hline

\begin{table*}[htbp]
    %\centering
    \footnotesize
    \begin{tabular}{|l|l|c|c|c|c|c|c|c|}
        \hline
        \multirow{2}{*}{\textbf{Task}} & \multirow{2}{*}{\textbf{Model}} & \multicolumn{7}{c|}{\textbf{Metrics}} \\ \cline{3-9} 
        & & \textbf{Val PPL} & \textbf{Bin0 PPL} & \textbf{Bin1 PPL} & \textbf{Bin2 PPL} & \textbf{Bin0 Acc} & \textbf{Bin1 Acc} & \textbf{Bin2 Acc} \\ \hline

		\multirow{5}{*}{\textbf{count-3}}
		& lstm & \textbf{1.04} & 1.06 & 1.54 & 1772.19 & 0.98 & 0.87 & 0.75 \\ \cline{2-9}
		& jm-10 & \textbf{1.04} & \textbf{1.05} & 1.18 & 1.67 & \textbf{0.99} & 0.90 & 0.76 \\ \cline{2-9}
		& jm-3.3.3 & \textbf{1.04} & \textbf{1.05} & 1.42 & 11.23 & \textbf{0.99} & 0.93 & 0.78 \\ \cline{2-9}
		& jm-hidden & \textbf{1.04} & \textbf{1.05} & \textbf{1.14} & \textbf{1.22} & \textbf{0.99} & \textbf{0.94} & 0.81 \\ \cline{2-9}
		& rns-3-3 & \textbf{1.04} & 1.06 & 1.35 & 155.42 & 0.98 & \textbf{0.94} & \textbf{0.85} \\ \hline

		\multirow{5}{*}{\textbf{marked-reverse-and-copy}}
		& lstm & 1.40 & 1.56 & 10.65 & \textbf{10.57} & 0.76 & 0.49 & 0.48 \\ \cline{2-9}
		& jm-10 & 1.33 & 1.42 & 3.55 & 182.41 & 0.80 & 0.50 & 0.43 \\ \cline{2-9}
		& jm-3.3.3 & 1.33 & 1.40 & 11.46 & 1060.93 & 0.80 & 0.50 & 0.31 \\ \cline{2-9}
		& jm-hidden & 1.32 & 1.49 & 103.81 & 12.71 & 0.79 & 0.48 & 0.48 \\ \cline{2-9}
		& rns-3-3 & \textbf{1.30} & \textbf{1.33} & \textbf{1.68} & 18.47 & \textbf{0.82} & \textbf{0.75} & \textbf{0.67} \\ \hline

		\multirow{5}{*}{\textbf{count-and-copy}}
		& lstm & 1.48 & 1.57 & 13.71 & 18.75 & 0.73 & 0.61 & 0.61 \\ \cline{2-9}
		& jm-10 & 1.39 & 1.61 & 368.71 & 10.42 & 0.76 & 0.60 & 0.59 \\ \cline{2-9}
		& jm-3.3.3 & 1.37 & 1.56 & \textbf{1.97} & 4.93 & 0.76 & 0.64 & \textbf{0.62} \\ \cline{2-9}
		& jm-hidden & 1.34 & 1.41 & 13.04 & \textbf{3.53} & 0.80 & 0.61 & 0.57 \\ \cline{2-9}
		& rns-3-3 & \textbf{1.31} & \textbf{1.38} & 2.87 & 19.00 & \textbf{0.81} & \textbf{0.69} & 0.55 \\ \hline

		\multirow{5}{*}{\textbf{marked-copy}}
		& lstm & 1.66 & 1.81 & 3.11 & 17.58 & 0.65 & 0.49 & 0.43 \\ \cline{2-9}
		& jm-10 & 1.63 & 1.77 & 14.13 & 16.13 & 0.66 & 0.40 & 0.30 \\ \cline{2-9}
		& jm-3.3.3 & 1.68 & 1.77 & \textbf{2.55} & \textbf{2.79} & 0.65 & \textbf{0.53} & \textbf{0.50} \\ \cline{2-9}
		& jm-hidden & 1.59 & 1.74 & 2.62 & 15.53 & 0.68 & 0.42 & 0.43 \\ \cline{2-9}
		& rns-3-3 & \textbf{1.47} & \textbf{1.51} & 11.05 & 2.82 & \textbf{0.73} & 0.50 & 0.32 \\ \hline

		\multirow{5}{*}{\textbf{unmarked-copy-diff-alpha}}
		& lstm & 1.61 & 1.68 & 2.66 & 17.51 & 0.68 & \textbf{0.56} & \textbf{0.52} \\ \cline{2-9}
		& jm-10 & \textbf{1.56} & \textbf{1.65} & 3.10 & 12.99 & \textbf{0.70} & 0.53 & 0.43 \\ \cline{2-9}
		& jm-3.3.3 & 1.57 & \textbf{1.65} & 3.00 & 14.73 & 0.68 & 0.52 & 0.39 \\ \cline{2-9}
		& jm-hidden & 1.58 & 1.70 & 2.40 & \textbf{10.02} & 0.68 & 0.48 & 0.46 \\ \cline{2-9}
		& rns-3-3 & \textbf{1.56} & 1.66 & \textbf{2.20} & 13.60 & 0.69 & 0.52 & 0.41 \\ \hline

		\multirow{5}{*}{\textbf{unmarked-reverse-and-copy}}
		& lstm & 1.66 & 1.84 & 3.32 & 3.93 & 0.67 & 0.54 & \textbf{0.51} \\ \cline{2-9}
		& jm-10 & \textbf{1.60} & \textbf{1.78} & 3.65 & 10.15 & \textbf{0.70} & \textbf{0.56} & \textbf{0.51} \\ \cline{2-9}
		& jm-3.3.3 & 1.64 & 1.83 & \textbf{3.24} & \textbf{3.66} & 0.68 & 0.54 & 0.50 \\ \cline{2-9}
		& jm-hidden & \textbf{1.60} & 1.81 & 4.21 & 11.75 & 0.69 & 0.54 & 0.50 \\ \cline{2-9}
		& rns-3-3 & \textbf{1.60} & \textbf{1.78} & 3.65 & 13.46 & 0.69 & 0.53 & 0.50 \\ \hline

		\multirow{5}{*}{\textbf{unmarked-copy}}
		& lstm & 1.85 & 1.92 & 2.61 & 3.19 & 0.60 & 0.54 & 0.51 \\ \cline{2-9}
		& jm-10 & 1.82 & 1.91 & 2.49 & 2.69 & 0.61 & 0.53 & \textbf{0.50} \\ \cline{2-9}
		& jm-3.3.3 & 1.83 & 1.92 & 2.20 & 2.25 & 0.61 & 0.53 & \textbf{0.50} \\ \cline{2-9}
		& jm-hidden & \textbf{1.76} & \textbf{1.86} & \textbf{2.16} & \textbf{2.18} & 0.63 & 0.55 & 0.51 \\ \cline{2-9}
		& rns-3-3 & 1.94 & 2.06 & 2.19 & 2.23 & \textbf{0.56} & \textbf{0.50} & \textbf{0.50} \\ \hline
        
    \end{tabular}
    \caption{Performance of different models on context free languages when all models components are fully trained (none). We test all models on 3 independent test set, slightly long test set bin 0 (N = $[40-100]$), mid range test (N=$[100-200]$)set bin1 and bin2 (N = $[200-400]$}
    \label{tab:tasks}
\end{table*}

\begin{table*}[htbp]
    %\centering
    \footnotesize
    \begin{tabular}{|l|l|c|c|c|c|c|c|c|}
        \hline
        \multirow{2}{*}{\textbf{Task}} & \multirow{2}{*}{\textbf{Model}} & \multicolumn{7}{c|}{\textbf{Metrics}} \\ \cline{3-9} 
        & & \textbf{Val PPL} & \textbf{Bin0 PPL} & \textbf{Bin1 PPL} & \textbf{Bin2 PPL} & \textbf{Bin0 Acc} & \textbf{Bin1 Acc} & \textbf{Bin2 Acc} \\ \hline
        
		\multirow{5}{*}{\textbf{count-3}}
		& lstm & 1.19 & 1.17 & \textbf{1.12} & \textbf{1.10} & 0.96 & \textbf{0.98} & 0.99 \\ \cline{2-9}
		& jm-10 & 1.10 & 1.12 & 1.32 & 1.43 & 0.97 & 0.90 & 0.79 \\ \cline{2-9}
		& jm-3.3.3 & \textbf{1.05} & \textbf{1.08} & 1.57 & 14.65 & \textbf{0.98} & 0.85 & 0.76 \\ \cline{2-9}
		& jm-hidden & 1.15 & 1.16 & 1.16 & 1.18 & 0.96 & \textbf{0.98} & 0.99 \\ \cline{2-9}
		& rns-3-3 & \textbf{1.05} & 1.21 & 136.65 & 126512.54 & 0.97 & 0.83 & \textbf{0.75} \\ \hline

		\multirow{5}{*}{\textbf{marked-reverse-and-copy}}
		& lstm & 2.38 & 2.35 & 2.23 & \textbf{2.18} & 0.49 & 0.50 & 0.50 \\ \cline{2-9}
		& jm-10 & 2.17 & 2.27 & 2.45 & 2.39 & 0.50 & 0.48 & 0.49 \\ \cline{2-9}
		& jm-3.3.3 & 1.83 & 1.92 & 10.43 & 12.49 & 0.66 & 0.50 & 0.50 \\ \cline{2-9}
		& jm-hidden & 2.27 & 2.26 & 2.21 & 2.19 & 0.56 & 0.53 & 0.51 \\ \cline{2-9}
		& rns-3-3 & \textbf{1.55} & \textbf{1.58} & \textbf{1.95} & 2.29 & \textbf{0.75} & \textbf{0.65} & \textbf{0.55} \\ \hline

		\multirow{5}{*}{\textbf{count-and-copy}}
		& lstm & 1.93 & 1.89 & \textbf{1.79} & \textbf{1.74} & 0.64 & \textbf{0.65} & \textbf{0.66} \\ \cline{2-9}
		& jm-10 & 1.70 & 1.79 & 15.26 & 10.83 & 0.63 & 0.50 & 0.42 \\ \cline{2-9}
		& jm-3.3.3 & 1.66 & 1.81 & 10.60 & 108.02 & 0.63 & 0.43 & 0.30 \\ \cline{2-9}
		& jm-hidden & 1.85 & 1.84 & \textbf{1.79} & \textbf{1.74} & 0.64 & \textbf{0.65} & \textbf{0.66} \\ \cline{2-9}
		& rns-3-3 & \textbf{1.65} & \textbf{1.71} & 1.95 & 13.67 & \textbf{0.65} & 0.57 & 0.53 \\ \hline

		\multirow{5}{*}{\textbf{marked-copy}}
		& lstm & 2.30 & 2.27 & 2.16 & \textbf{2.12} & \textbf{0.49} & 0.49 & \textbf{0.50} \\ \cline{2-9}
		& jm-10 & 2.11 & 2.20 & 2.37 & 17.03 & \textbf{0.49} & 0.49 & \textbf{0.50} \\ \cline{2-9}
		& jm-3.3.3 & 2.13 & 2.18 & 2.37 & 2.37 & \textbf{0.49} & \textbf{0.50} & \textbf{0.50} \\ \cline{2-9}
		& jm-hidden & 2.24 & 2.22 & \textbf{2.15} & \textbf{2.12} & \textbf{0.49} & 0.49 & \textbf{0.50} \\ \cline{2-9}
		& rns-3-3 & \textbf{2.07} & \textbf{2.16} & 2.90 & 13.93 & \textbf{0.49} & 0.49 & \textbf{0.50} \\ \hline

		\multirow{5}{*}{\textbf{unmarked-copy-diff-alpha}}
		& lstm & 2.28 & 2.25 & \textbf{2.16} & \textbf{2.12} & 0.49 & 0.49 & \textbf{0.50} \\ \cline{2-9}
		& jm-10 & 2.04 & 2.14 & 13.00 & 10846.00 & 0.55 & 0.46 & 0.36 \\ \cline{2-9}
		& jm-3.3.3 & \textbf{2.03} & \textbf{2.12} & 19.82 & 2.62 & \textbf{0.56} & 0.49 & \textbf{0.50} \\ \cline{2-9}
		& jm-hidden & 2.21 & 2.21 & 2.19 & 2.17 & 0.51 & \textbf{0.50} & \textbf{0.50} \\ \cline{2-9}
		& rns-3-3 & 2.08 & 11.02 & 2.31 & 181.99 & 0.50 & 0.49 & \textbf{0.50} \\ \hline

		\multirow{5}{*}{\textbf{unmarked-reverse-and-copy}}
		& lstm & 2.15 & 2.13 & \textbf{2.08} & \textbf{2.06} & 0.51 & 0.50 & 0.50 \\ \cline{2-9}
		& jm-10 & 2.05 & 2.12 & 2.58 & 2.81 & 0.55 & 0.50 & 0.49 \\ \cline{2-9}
		& jm-3.3.3 & 2.04 & 2.10 & 2.47 & 2.71 & 0.56 & 0.52 & 0.51 \\ \cline{2-9}
		& jm-hidden & 2.10 & 2.11 & 2.12 & 2.13 & 0.51 & 0.51 & 0.50 \\ \cline{2-9}
		& rns-3-3 & \textbf{1.87} & \textbf{1.97} & 2.17 & 2.20 & \textbf{0.62} & \textbf{0.55} & \textbf{0.52} \\ \hline

		\multirow{5}{*}{\textbf{unmarked-copy}}
		& lstm & 2.15 & 2.13 & \textbf{2.08} & \textbf{2.06} & 0.49 & \textbf{0.50} & \textbf{0.50} \\ \cline{2-9}
		& jm-10 & \textbf{2.06} & 2.11 & 2.22 & 2.29 & \textbf{0.55} & \textbf{0.50} & \textbf{0.50} \\ \cline{2-9}
		& jm-3.3.3 & 2.03 & \textbf{2.09} & 2.20 & 2.27 & \textbf{0.55} & \textbf{0.50} & \textbf{0.50} \\ \cline{2-9}
		& jm-hidden & 2.11 & 2.11 & 2.14 & 2.16 & 0.51 & \textbf{0.50} & 0.50 \\ \cline{2-9}
		& rns-3-3 & 2.11 & 2.12 & 2.16 & 2.19 & 0.49 & \textbf{0.50} & \textbf{0.50} \\ \hline
        
    \end{tabular}
    \caption{Performance of different models on context free languages when only parameter belonging to controller are trainable (mode = c). We test all models on 3 independent test set, slightly long test set bin 0 (N = $[40-100]$), mid range test (N=$[100-200]$)set bin1 and bin2 (N = $[200-400]$.}
    \label{tab:tasks}
\end{table*}

\begin{table*}[htbp]
    %\centering
    \footnotesize
    \begin{tabular}{|l|l|c|c|c|c|c|c|c|}
        \hline
        \multirow{2}{*}{\textbf{Task}} & \multirow{2}{*}{\textbf{Model}} & \multicolumn{7}{c|}{\textbf{Metrics}} \\ \cline{3-9} 
        & & \textbf{Val PPL} & \textbf{Bin0 PPL} & \textbf{Bin1 PPL} & \textbf{Bin2 PPL} & \textbf{Bin0 Acc} & \textbf{Bin1 Acc} & \textbf{Bin2 Acc} \\ \hline
        
		\multirow{5}{*}{\textbf{count-3}}
		& lstm & \textbf{1.04} & \textbf{1.05} & 1.28 & 1.88 & \textbf{0.99} & 0.90 & 0.75 \\ \cline{2-9}
		& jm-10 & \textbf{1.04} & \textbf{1.05} & 1.21 & 1.70 & 0.98 & 0.86 & 0.76 \\ \cline{2-9}
		& jm-3.3.3 & \textbf{1.04} & 1.06 & 1.25 & \textbf{1.69} & 0.98 & 0.92 & 0.78 \\ \cline{2-9}
		& jm-hidden & \textbf{1.04} & \textbf{1.05} & \textbf{1.19} & 1.70 & \textbf{0.99} & \textbf{0.99} & \textbf{0.97} \\ \cline{2-9}
		& rns-3-3 & \textbf{1.04} & \textbf{1.05} & 1.36 & 2.87 & 0.98 & 0.85 & 0.72 \\ \hline

		\multirow{5}{*}{\textbf{marked-reverse-and-copy}}
		& lstm & 1.38 & 1.55 & \textbf{11.64} & \textbf{3.56} & 0.77 & 0.48 & 0.48 \\ \cline{2-9}
		& jm-10 & 1.38 & 1.51 & 20.13 & 10.30 & 0.77 & 0.49 & 0.47 \\ \cline{2-9}
		& jm-3.3.3 & 1.41 & 1.54 & 12.47 & 11.30 & 0.76 & 0.49 & 0.46 \\ \cline{2-9}
		& jm-hidden & 1.37 & \textbf{1.49} & 134.47 & 13.62 & \textbf{0.78} & \textbf{0.56} & \textbf{0.50} \\ \cline{2-9}
		& rns-3-3 & \textbf{1.36} & 1.52 & 23.70 & 233.30 & \textbf{0.78} & 0.47 & 0.41 \\ \hline

		\multirow{5}{*}{\textbf{count-and-copy}}
		& lstm & 1.43 & \textbf{1.50} & \textbf{1.93} & 16.86 & \textbf{0.76} & 0.63 & 0.54 \\ \cline{2-9}
		& jm-10 & 1.45 & 1.53 & 2.01 & \textbf{2.39} & \textbf{0.76} & \textbf{0.64} & \textbf{0.56} \\ \cline{2-9}
		& jm-3.3.3 & 1.49 & 1.56 & 2.41 & 1266.62 & 0.73 & 0.60 & 0.53 \\ \cline{2-9}
		& jm-hidden & \textbf{1.40} & 1.57 & 131.80 & 115.32 & \textbf{0.76} & 0.60 & 0.54 \\ \cline{2-9}
		& rns-3-3 & 1.43 & 1.58 & 16.84 & 175.87 & 0.74 & 0.59 & 0.53 \\ \hline

		\multirow{5}{*}{\textbf{marked-copy}}
		& lstm & 1.65 & \textbf{1.77} & 2.73 & 10.88 & \textbf{0.67} & 0.49 & 0.42 \\ \cline{2-9}
		& jm-10 & 1.75 & 1.81 & 2.32 & 2.92 & 0.63 & \textbf{0.50} & 0.36 \\ \cline{2-9}
		& jm-3.3.3 & 1.76 & 1.85 & \textbf{2.28} & \textbf{2.26} & 0.62 & 0.49 & \textbf{0.50} \\ \cline{2-9}
		& jm-hidden & \textbf{1.63} & 1.89 & 159.06 & 103370.22 & 0.63 & 0.47 & \textbf{0.50} \\ \cline{2-9}
		& rns-3-3 & 1.79 & 2.02 & 18.16 & 2.82 & 0.57 & 0.38 & 0.32 \\ \hline

		\multirow{5}{*}{\textbf{unmarked-copy-diff-alpha}}
		& lstm & 1.62 & \textbf{1.69} & \textbf{2.26} & 1103.65 & 0.67 & 0.54 & 0.47 \\ \cline{2-9}
		& jm-10 & \textbf{1.59} & 1.72 & 2.49 & 1242.00 & \textbf{0.68} & 0.51 & 0.39 \\ \cline{2-9}
		& jm-3.3.3 & 1.63 & 1.72 & 2.27 & 17.13 & 0.67 & \textbf{0.55} & 0.46 \\ \cline{2-9}
		& jm-hidden & 1.60 & 1.70 & 2.32 & 17.39 & 0.67 & 0.54 & \textbf{0.51} \\ \cline{2-9}
		& rns-3-3 & 1.71 & 1.82 & 13.09 & \textbf{13.72} & 0.65 & 0.52 & 0.49 \\ \hline

		\multirow{5}{*}{\textbf{unmarked-reverse-and-copy}}
		& lstm & 1.64 & \textbf{1.81} & 3.12 & \textbf{10.05} & 0.68 & 0.54 & 0.50 \\ \cline{2-9}
		& jm-10 & 1.68 & 1.86 & \textbf{3.02} & 10.41 & 0.66 & 0.54 & 0.50 \\ \cline{2-9}
		& jm-3.3.3 & 1.68 & 1.87 & 3.21 & 11.97 & 0.67 & 0.53 & 0.50 \\ \cline{2-9}
		& jm-hidden & 1.65 & 1.83 & 3.45 & 14.59 & 0.67 & 0.53 & 0.50 \\ \cline{2-9}
		& rns-3-3 & \textbf{1.61} & 1.82 & 3.29 & 11.07 & \textbf{0.69} & \textbf{0.56} & \textbf{0.51} \\ \hline

		\multirow{5}{*}{\textbf{unmarked-copy}}
		& lstm & 1.85 & 1.93 & \textbf{2.19} & \textbf{2.23} & 0.60 & \textbf{0.54} & \textbf{0.51} \\ \cline{2-9}
		& jm-10 & 1.82 & 1.92 & 2.44 & 2.67 & 0.61 & \textbf{0.54} & \textbf{0.51} \\ \cline{2-9}
		& jm-3.3.3 & 1.84 & 1.94 & 2.21 & 14.98 & 0.61 & 0.53 & 0.50 \\ \cline{2-9}
		& jm-hidden & \textbf{1.81} & \textbf{1.90} & 2.68 & 10.87 & \textbf{0.62} & \textbf{0.54} & 0.50 \\ \cline{2-9}
		& rns-3-3 & 1.85 & 1.94 & 2.27 & 2.38 & 0.60 & 0.53 & 0.50 \\ \hline
        
    \end{tabular}
    \caption{Performance of different models on context free languages when only parameter belonging to memory are trainable (mode = m). We test all models on 3 independent test set, slightly long test set bin 0 (N = $[40-100]$), mid range test (N=$[100-200]$)set bin1 and bin2 (N = $[200-400]$}
    \label{tab:tasks}
\end{table*}

\begin{table*}[htbp]
    %\centering
    \footnotesize
    \begin{tabular}{|l|l|c|c|c|c|c|c|c|}
        \hline
        \multirow{2}{*}{\textbf{Task}} & \multirow{2}{*}{\textbf{Model}} & \multicolumn{7}{c|}{\textbf{Metrics}} \\ \cline{3-9} 
        & & \textbf{Val PPL} & \textbf{Bin0 PPL} & \textbf{Bin1 PPL} & \textbf{Bin2 PPL} & \textbf{Bin0 Acc} & \textbf{Bin1 Acc} & \textbf{Bin2 Acc} \\ \hline
        
		\multirow{5}{*}{\textbf{count-3}}
		& lstm & 1.18 & 1.17 & \textbf{1.12} & \textbf{1.10} & \textbf{0.96} & \textbf{0.98} & \textbf{0.99} \\ \cline{2-9}
		& jm-10 & 1.17 & 1.17 & 1.20 & 1.33 & \textbf{0.96} & 0.96 & 0.83 \\ \cline{2-9}
		& jm-3.3.3 & 1.16 & 1.16 & 1.15 & 1.18 & \textbf{0.96} & \textbf{0.98} & \textbf{0.99} \\ \cline{2-9}
		& jm-hidden & \textbf{1.15} & \textbf{1.15} & \textbf{1.12} & 1.11 & \textbf{0.96} & \textbf{0.98} & \textbf{0.99} \\ \cline{2-9}
		& rns-3-3 & 1.19 & 1.17 & \textbf{1.12} & \textbf{1.10} & \textbf{0.96} & \textbf{0.98} & \textbf{0.99} \\ \hline

		\multirow{5}{*}{\textbf{marked-reverse-and-copy}}
		& lstm & 2.38 & 2.35 & \textbf{2.23} & \textbf{2.18} & 0.49 & \textbf{0.50} & \textbf{0.50} \\ \cline{2-9}
		& jm-10 & \textbf{2.34} & \textbf{2.33} & 2.26 & 2.25 & \textbf{0.50} & \textbf{0.50} & \textbf{0.50} \\ \cline{2-9}
		& jm-3.3.3 & \textbf{2.34} & \textbf{2.33} & 2.28 & 2.30 & \textbf{0.50} & \textbf{0.50} & \textbf{0.50} \\ \cline{2-9}
		& jm-hidden & 2.37 & 2.34 & \textbf{2.23} & 2.19 & \textbf{0.50} & \textbf{0.50} & \textbf{0.50} \\ \cline{2-9}
		& rns-3-3 & 2.39 & 2.36 & \textbf{2.23} & \textbf{2.18} & 0.49 & \textbf{0.50} & \textbf{0.50} \\ \hline

		\multirow{5}{*}{\textbf{count-and-copy}}
		& lstm & 1.93 & 1.89 & \textbf{1.79} & \textbf{1.74} & \textbf{0.64} & \textbf{0.65} & \textbf{0.66} \\ \cline{2-9}
		& jm-10 & \textbf{1.89} & 1.87 & 1.82 & 1.81 & \textbf{0.64} & \textbf{0.65} & \textbf{0.66} \\ \cline{2-9}
		& jm-3.3.3 & \textbf{1.89} & \textbf{1.86} & 1.81 & 1.80 & \textbf{0.64} & \textbf{0.65} & \textbf{0.66} \\ \cline{2-9}
		& jm-hidden & 1.90 & 1.87 & 1.80 & 1.76 & \textbf{0.64} & \textbf{0.65} & \textbf{0.66} \\ \cline{2-9}
		& rns-3-3 & 1.93 & 1.90 & \textbf{1.79} & \textbf{1.74} & \textbf{0.64} & \textbf{0.65} & \textbf{0.66} \\ \hline

		\multirow{5}{*}{\textbf{marked-copy}}
		& lstm & 2.29 & 2.27 & \textbf{2.16} & \textbf{2.12} & \textbf{0.49} & \textbf{0.49} & \textbf{0.50} \\ \cline{2-9}
		& jm-10 & \textbf{2.26} & 2.25 & 2.19 & 2.18 & \textbf{0.49} & \textbf{0.49} & \textbf{0.50} \\ \cline{2-9}
		& jm-3.3.3 & \textbf{2.26} & \textbf{2.24} & 2.23 & 2.29 & \textbf{0.49} & \textbf{0.49} & \textbf{0.50} \\ \cline{2-9}
		& jm-hidden & 2.28 & 2.26 & 2.17 & \textbf{2.12} & \textbf{0.49} & \textbf{0.49} & \textbf{0.50} \\ \cline{2-9}
		& rns-3-3 & 2.29 & 2.26 & \textbf{2.16} & \textbf{2.12} & \textbf{0.49} & \textbf{0.49} & \textbf{0.50} \\ \hline

		\multirow{5}{*}{\textbf{unmarked-copy-diff-alpha}}
		& lstm & 2.28 & 2.24 & \textbf{2.16} & \textbf{2.12} & 0.49 & 0.49 & \textbf{0.50} \\ \cline{2-9}
		& jm-10 & 2.24 & \textbf{2.22} & 2.19 & 2.20 & 0.49 & 0.49 & \textbf{0.50} \\ \cline{2-9}
		& jm-3.3.3 & \textbf{2.23} & \textbf{2.22} & 2.18 & 2.18 & 0.49 & \textbf{0.50} & \textbf{0.50} \\ \cline{2-9}
		& jm-hidden & 2.26 & 2.23 & \textbf{2.16} & \textbf{2.12} & \textbf{0.50} & \textbf{0.50} & \textbf{0.50} \\ \cline{2-9}
		& rns-3-3 & 2.28 & 2.24 & \textbf{2.16} & \textbf{2.12} & 0.49 & 0.49 & \textbf{0.50} \\ \hline

		\multirow{5}{*}{\textbf{unmarked-reverse-and-copy}}
		& lstm & 2.15 & 2.13 & \textbf{2.08} & \textbf{2.06} & \textbf{0.51} & 0.50 & \textbf{0.50} \\ \cline{2-9}
		& jm-10 & \textbf{2.11} & \textbf{2.11} & 2.12 & 2.17 & \textbf{0.51} & 0.50 & \textbf{0.50} \\ \cline{2-9}
		& jm-3.3.3 & \textbf{2.11} & \textbf{2.11} & 2.10 & 2.11 & \textbf{0.51} & \textbf{0.51} & \textbf{0.50} \\ \cline{2-9}
		& jm-hidden & 2.13 & 2.12 & \textbf{2.08} & \textbf{2.06} & \textbf{0.51} & \textbf{0.51} & \textbf{0.50} \\ \cline{2-9}
		& rns-3-3 & 2.15 & 2.13 & \textbf{2.08} & \textbf{2.06} & \textbf{0.51} & 0.50 & \textbf{0.50} \\ \hline

		\multirow{5}{*}{\textbf{unmarked-copy}}
		& lstm & 2.15 & 2.13 & \textbf{2.08} & \textbf{2.06} & 0.49 & \textbf{0.50} & \textbf{0.50} \\ \cline{2-9}
		& jm-10 & \textbf{2.12} & \textbf{2.11} & 2.13 & 2.22 & \textbf{0.50} & \textbf{0.50} & \textbf{0.50} \\ \cline{2-9}
		& jm-3.3.3 & \textbf{2.12} & \textbf{2.11} & 2.10 & 2.11 & \textbf{0.50} & \textbf{0.50} & \textbf{0.50} \\ \cline{2-9}
		& jm-hidden & 2.14 & 2.12 & \textbf{2.08} & \textbf{2.06} & \textbf{0.50} & \textbf{0.50} & \textbf{0.50} \\ \cline{2-9}
		& rns-3-3 & 2.15 & 2.13 & \textbf{2.08} & \textbf{2.06} & \textbf{0.50} & \textbf{0.50} & \textbf{0.50} \\ \hline
        
    \end{tabular}
    \caption{Performance of different models on context free languages when only classifier is trained; whereas parameter belonging to controller and memory are fixed (mode = cm). We test all models on 3 independent test set, slightly long test set bin 0 (N = $[40-100]$), mid range test (N=$[100-200]$)set bin1 and bin2 (N = $[200-400]$}
    \label{tab:tasks}
\end{table*}

\section{Appendix C: Evaluating Stability in Practice}

In this section we provide various ways the stability in stack-augmented Recurrent Neural Networks (RNNs) can be practically assessed by focusing on the model’s error behavior, generalization to longer sequences, and resilience to perturbations. The following conditions are essential for determining whether the system remains stable across varying input scenarios.

\subsection{Error Bounds Across Sequence Lengths}

A stable system should maintain a bounded error for sequences of varying lengths. Let the loss function be \( \text{Loss}(f_\theta(x), y) \) for input \( x \) and target \( y \). Ideally, the error should satisfy:

\[
\left| \text{Loss}(f_\theta(x), y) \right| \leq C \quad \forall x \in \Sigma^* \text{ where } T_{\min} \leq |x| \leq T_{\max},
\]

where \( T_{\min} \) and \( T_{\max} \) define the minimum and maximum sequence lengths tested. The goal is to ensure that the error remains within a predefined bound \( C > 0 \) across sequences of different lengths.

\subsection{Consistency of Error Across Sequence Lengths}

The error should be consistent across sequences of varying lengths. Variance in error as a function of sequence length should be minimal. Formally, the stability criterion is:

\[
\text{Var}\left( \text{Loss}(f_\theta(x), y) \right) \text{ should be minimized across varying } |x|.
\]

Low variance in error indicates that the model’s behavior is predictable and stable regardless of the input sequence length.

\subsection{Performance Degradation Over Long Sequences}

Stability is also determined by the model’s ability to handle sequences significantly longer than those seen during training. Let \( x_{\text{long}} \) and \( x_{\text{short}} \) be sequences of longer and shorter lengths, respectively. Ideally, the model should satisfy:

\[
\text{Loss}(f_\theta(x_{\text{long}}), y_{\text{long}}) \approx \text{Loss}(f_\theta(x_{\text{short}}), y_{\text{short}}).
\]

Testing on sequences 2-4 times longer than those seen during training helps reveal whether the system can generalize to more complex scenarios without significant error growth.

\subsection{Generalization to Unseen Sequence Patterns}

The system should generalize to novel input sequences that conform to the language \( L \) but are not explicitly encountered during training. The criterion is that the model’s error on unseen patterns should be similar to that on known patterns:

\[
\left| \text{Loss}(f_\theta(x_{\text{new}}), y_{\text{new}}) - \text{Loss}(f_\theta(x_{\text{known}}), y_{\text{known}}) \right| \leq \delta,
\]

where \( x_{\text{new}} \) and \( x_{\text{known}} \) represent novel and known sequences, respectively, and \( \delta > 0 \) is a small acceptable difference in error.

\subsection{Analyzing Error Growth}

The growth of error as sequence length increases is a critical measure of stability. For a stable system, error growth should be sub-linear or, at most, linear:

\[
\left| \text{Loss}(f_\theta(x), y) \right| \leq C_1 \cdot |x| + C_2,
\]

where \( C_1 \) and \( C_2 \) are constants. Super-linear growth (e.g., quadratic or exponential) indicates instability and should be avoided.

\subsection{Robustness to Perturbations}

A stable model should be resilient to small perturbations in the input. For a small perturbation \( \epsilon \) added to the input sequence \( x \), the error difference should be minimal:

\[
\left| \text{Loss}(f_\theta(x + \epsilon), y) - \text{Loss}(f_\theta(x), y) \right| \leq \delta,
\]

where \( \delta > 0 \) is an acceptable margin for error change. High sensitivity to perturbations suggests instability.

\end{document}